\documentclass[preprint]{statml}

\usepackage[round,authoryear]{natbib}
\usepackage{longtable}
\usepackage{makecell}
\usepackage{algorithm}
\usepackage{algorithmic}
\usepackage{float}
\usepackage{mathrsfs}
\usepackage{bm}
\usepackage{empheq}

\usetikzlibrary{arrows.meta,positioning,calc,shapes.geometric,backgrounds,fadings,decorations.pathreplacing,fit}
\definecolor{learnblue}{HTML}{1F77B4}
\definecolor{learngreen}{HTML}{2CA02C}
\definecolor{learnorange}{HTML}{FF7F0E}
\definecolor{learnpurple}{HTML}{9467BD}
\definecolor{querypurple}{HTML}{6A3D9A}
\definecolor{panelstroke}{HTML}{C9CCD6}
\definecolor{panelfill}{HTML}{FBFCFE}
\definecolor{iconbluebg}{HTML}{DCE9F6}
\definecolor{icongreenbg}{HTML}{DDF0DD}
\definecolor{iconorangebg}{HTML}{FCE6CE}
\definecolor{iconpurplebg}{HTML}{ECE2F4}
\definecolor{peoplegray}{HTML}{BFC4CE}
\definecolor{peopleblue}{HTML}{AEC7E8}
\definecolor{arrowcol}{HTML}{5B6270}
\definecolor{baseline}{HTML}{B7BDC8}
\definecolor{bandpurple}{HTML}{C9B8E0}
\tikzset{
  person/.pic={
    \fill (0,0.165) circle (0.082);
    \fill[rounded corners=1.4pt]
      (-0.118,-0.205) -- (-0.118,-0.05)
      .. controls (-0.118,0.075) and (0.118,0.075) .. (0.118,-0.05)
      -- (0.118,-0.205) -- cycle;
  }
}
\tikzset{
  checkmark/.pic={
    \draw[white,line width=1.1pt,line cap=round,line join=round]
      (-0.085,0.0) -- (-0.02,-0.075) -- (0.1,0.085);
  }
}
\tikzset{
  anntree/.pic={
    \draw[arrowcol!55,line width=0.5pt]
      (0,0.62)--(-0.5,0.28) (0,0.62)--(0.5,0.28)
      (-0.5,0.22)--(-0.72,-0.12) (-0.5,0.22)--(-0.28,-0.12)
      (0.5,0.22)--(0.28,-0.12) (0.5,0.22)--(0.72,-0.12);
    \fill[learnblue!75] (0,0.62) circle (0.1);
    \fill[learnblue!50] (-0.5,0.25) circle (0.085);
    \fill[learnblue!50] (0.5,0.25) circle (0.085);
    \fill[learnblue!32] (-0.72,-0.15) circle (0.075);
    \fill[learnblue!32] (-0.28,-0.15) circle (0.075);
    \fill[learnblue!32] (0.28,-0.15) circle (0.075);
    \fill[learnblue!32] (0.72,-0.15) circle (0.075);
  }
}
\definecolor{weightblue}{HTML}{1F77B4}
\definecolor{eventred}{HTML}{D62728}
\definecolor{riskgreen}{HTML}{2CA02C}
\definecolor{proxorange}{HTML}{FF7F0E}
\definecolor{hazardpurple}{HTML}{9467BD}
\definecolor{statmlProjectLink}{HTML}{D4512A}
\newcommand{\ind}{\mathbf{1}}
\newcommand{\appref}[1]{Appendix~\ref{#1}}

\title{KESurv: A Kernel Ensemble Method\texorpdfstring{\\}{ }for Patient-Specific Survival Prediction}
\shorttitle{KESurv}
\addauthor{Rahul Goswami}{1}
\addaffil{1}{Department of Mathematics, Indian Institute of Technology Guwahati}
\addlogo[https://www.iitg.ac.in/]{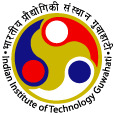}
\correspondence{rahul@statml.in}
\addresource{\textcolor{statmlProjectLink}{\faGlobe}}{
  \textcolor{statmlProjectLink}{\bfseries Project page}}{https://kesurv.statml.in}

\begin{document}
\begin{abstract}
 Predicting patient-specific survival functions is crucial for clinicians in making informed decisions about patient care and treatment strategies. Among the various models available, the Survival Forest has demonstrated significant effectiveness in numerous scenarios. In this work, we propose an ensemble method that leverages the strengths of the Survival Forest as the master model, complemented by several base models. This ensemble incorporates the Beran estimator, a type of kernel estimator, to enhance predictions of patient-specific survival curves. We evaluated the performance of our proposed model using four distinct healthcare datasets. The results highlight the superiority of our ensemble method over baseline models in both calibration and ranking across most datasets. The findings suggest that our approach offers a more accurate and reliable estimation of patient-specific survival functions, providing a valuable tool for clinical decision-making.
\end{abstract}
\maketitle

\section{Introduction}
Survival analysis is a fundamental tool in medical and epidemiological research for modeling the time until an event of interest occurs. Such events may include mortality, disease progression, relapse, or other clinically meaningful outcomes. Let $T$ denote the event time and $X$ denote a vector of patient covariates. The conditional survival function
$
    S(t\mid x)=\mathbb{P}(T>t\mid X=x)
$
describes the probability that a patient with features $x$ survives beyond time $t$. Estimating this function yields Patient-Specific Survival Curves (PSSCs), which provide individualized prognostic information and support data-driven clinical decision-making \citep{cox1972regression,kalbfleisch2011statistical,frank2015regression}.
PSSCs have been widely used across medical domains where patient heterogeneity plays an important role. For example, gene expression signatures have been used to improve survival prediction in breast cancer \citep{van2002gene}, while multi-site studies in lung cancer have demonstrated the prognostic value of gene expression-based survival models \citep{director2008gene}. Survival analysis has also been used in colorectal cancer to distinguish disease-free survival from overall survival when evaluating adjuvant therapies \citep{sargent2005disease}. These examples illustrate the practical importance of accurate patient-specific survival estimation in clinical prognosis and treatment planning.
A large class of survival models estimates PSSCs by imposing structural assumptions on the conditional hazard function. Denoting the hazard by
$
    \lambda(t\mid x)=-\frac{d}{dt}\log S(t\mid x),
$
the Cox proportional hazards model assumes
$
    \lambda(t\mid x)=\lambda_0(t)\exp(\beta^\top x),
$
where $\lambda_0(t)$ is the baseline hazard and $\beta$ is a vector of regression coefficients \citep{cox1972regression}. After estimating $\lambda_0(t)$ and $\beta$ from training data, the corresponding survival curve can be recovered through
$
    S(t\mid x)=\exp\left(-\int_0^t \lambda(s\mid x)\,ds\right).
$
Other classical approaches, such as the Aalen additive model \citep{aalen1989linear} and accelerated failure time models \citep{kalbfleisch2011statistical}, similarly rely on structural assumptions about the relationship between covariates and survival outcomes.
More recently, survival analysis has benefited from machine learning methods that relax some of these assumptions. For instance, DeepSurv replaces the linear Cox risk function with a nonlinear function learned by a neural network \citep{katzman2018deepsurv}, while hierarchical generative models have been used to directly model patient-specific survival distributions \citep{ranganath2016deep}. Ensemble-based approaches, including Random Survival Forests (RSF) \citep{ishwaran2008random} and gradient-boosted survival models based on XGBoost \citep{chen2016xgboost}, have also become popular due to their ability to capture nonlinear covariate effects and interactions.
Another important direction is nonparametric survival estimation. The Beran estimator extends the Kaplan--Meier estimator by incorporating covariate information through local neighborhoods or kernel weights. While the Kaplan--Meier estimator estimates the marginal survival function
$
    S_{\mathrm{marg}}(t)=\mathbb{P}(T>t),
$
the Beran estimator estimates the conditional survival function $S(t\mid x)$ by giving higher weight to training subjects whose feature vectors are close to the test point $x$. This approach avoids strong parametric assumptions on the survival function and has been studied theoretically in settings such as nearest-neighbor survival estimation \citep{chen2019nearest}.
Despite their flexibility, nearest-neighbor and kernel-based survival estimators often face difficulties in high-dimensional settings. As the feature dimension increases, local neighborhoods become sparse, distances between observations become less discriminative, and kernel weights become sensitive to bandwidth choices. Consequently, local averaging may become unstable: small neighborhoods can lead to high variance, while larger neighborhoods can oversmooth the survival curve and increase bias. These challenges motivate methods that replace raw-feature proximity with a more informative notion of similarity.
In this study, we introduce Kernel Ensemble Survival (KESurv), a survival analysis framework that combines adaptive neighborhood construction with nonparametric survival estimation. Instead of defining proximity directly in the original feature space, KESurv uses a prediction-induced proximity obtained from a collection of diverse survival learners, inspired by the Combined Regression Aggregation idea of \citep{biau2016cobra}. This learned proximity acts as a collective indicator of similarity between training observations and a test patient. The resulting neighborhood is then used with survival estimators that produce patient-specific survival probabilities.

Several survival methods combine learned similarities, neighborhood construction, and local survival aggregation. \citet{chen2019nearest} studies nearest-neighbor and kernel survival estimators, including an adaptive kernel learned using random survival forests. Survival Kernets \citep{chen2024survival} constructs an $\epsilon$-net of representative exemplars in a learned embedding space, groups training observations around these exemplars, and combines nearby cluster summaries through truncated kernel weights to estimate survival probabilities. It also distinguishes the split construction covered by its theoretical guarantee from implementations evaluated with and without splitting. Area-norm COBRA \citep{goswami2023area} compares predicted survival curves through an area-based dissimilarity, selects reference observations by consensus across base learners, and applies Kaplan–Meier aggregation, using disjoint cohorts for base-learner fitting and aggregation. These methods provide related approaches to constructing and weighting local information for survival estimation. KESurv uses survival-curve consensus selection to filter a forest-induced adaptive neighborhood, retaining the forest-induced weights within a weighted Kaplan–Meier estimator. The distinction between our empirical no-split construction and the
honest/cross-fitted construction considered in our theoretical analysis
is clarified in Section~\ref{sec:experimental_setup} and
Appendix~\ref{app:theory}. The main contributions are:
(i) We propose KESurv, an ensemble-based method for estimating
Patient-Specific Survival Curves;
(ii) We use a COBRA-style survival-curve agreement selector that filters
an existing survival-forest-induced adaptive nearest-neighbor (ANN)
neighborhood, retaining a training observation only when a sufficient
number of diverse base survival learners predict a survival curve close
to that of the query patient;
(iii) We combine this consensus selector with the forest-induced
adaptive kernel inside a weighted Kaplan--Meier estimator to improve
patient-specific survival prediction; and
(iv) We evaluate the proposed method on four datasets and compare it
against standard survival baselines, including Random Survival Forests
and adaptive nearest-neighbor methods. The remainder of this paper is organized as follows.
Section~\ref{background} reviews survival analysis and the Beran estimator.
Section~\ref{proposed_methods} presents the proposed Kernel Ensemble
Survival method. Section~\ref{experiments} describes the datasets,
baseline models, experimental setup, and results.
Finally, Section~\ref{conc} concludes the paper and discusses future
research directions.
\section{Background}
\label{background}
Survival analysis studies time-to-event outcomes in settings where the event of interest may not be observed for every subject. Let $T\in\mathbb{R}_{+}$ denote the true event time and let $X\in\mathcal{X}\subseteq\mathbb{R}^{p}$ denote a vector of covariates. For a subject with covariates $x$, the central object of interest is the conditional survival function $S(t\mid x)$, which gives the probability that the subject survives beyond time $t$. Estimating this function for each individual produces a Patient-Specific Survival Curve (PSSC), which is the main predictive object considered in this paper.
In medical and epidemiological datasets, the true event time $T$ is often right-censored. Right censoring occurs when the event has not been observed by the end of follow-up, so the exact event time is unknown. To represent this situation, let $C$ denote the censoring time. Rather than observing $T$ directly for every subject, we observe the follow-up time $Y=\min(T,C)$ and an event indicator $\delta$, where $\delta=1$ indicates that the event was observed and $\delta=0$ indicates that the observation was censored. The basic notation used throughout this paper is summarized in Eq.~\ref{eq:survival_notation_panel}.

\begin{subequations}
\label{eq:survival_notation_panel}
\noindent
\begin{minipage}{0.48\linewidth}
  \begin{align}
    \label{eq:conditional_survival}
    S(t\mid x) &= \mathbb{P}(T>t\mid X=x) \\[1em]
    \label{eq:observed_dataset}
    \mathcal{D}_n &= \{(X_i,Y_i,\delta_i)\}_{i=1}^{n}
  \end{align}
\end{minipage}
\hfill
\begin{minipage}{0.48\linewidth}
  \begin{align}
    \label{eq:observed_time}
    Y &= \min(T,C), \quad \delta=\ind\{T\leq C\} \\[1em]
    \label{eq:conditional_independent_censoring}
    T &\perp C\mid X
  \end{align}
\end{minipage}
\vspace{5mm}
\end{subequations}

The dataset $\mathcal{D}_n$ consists of $n$ independent training subjects, where each subject is represented by a covariate vector $X_i$, an observed time $Y_i$, and an event indicator $\delta_i$. The assumption $T\perp C\mid X$ is the standard conditional independent censoring assumption. It states that, after conditioning on the observed covariates, the censoring mechanism does not provide additional information about the event time. This assumption is important because it allows the conditional survival curve $S(t\mid x)$ to be estimated from censored observations.
\subsection{Kernel Survival Estimator}
The Kaplan--Meier estimator estimates a marginal survival curve without using covariate information. In contrast, the Beran estimator extends the Kaplan--Meier idea to the conditional setting by assigning larger weights to training observations whose covariates are close to a query point $x$. This produces a local survival estimator for the specific subject represented by $x$.
To construct the kernel survival estimator, we require a nonnegative kernel function $K$, a bandwidth parameter $h>0$, and a distance or dissimilarity function $\rho$. The kernel controls how much influence a training subject has on the survival estimate for $x$, while the bandwidth controls the size of the local neighborhood. A small bandwidth leads to a more local estimator with potentially lower bias but higher variance. A large bandwidth includes more observations and may reduce variance, but can oversmooth subject-specific survival patterns.
The main ingredients of the kernel estimator are summarized in Eq.~\ref{eq:kernel_quantity_panel}. The blue term represents the normalized local kernel weight, the red term represents the weighted number of observed deaths, and the green term represents the weighted number of subjects at risk.

\begin{subequations}
\label{eq:kernel_quantity_panel}
\noindent
\begin{minipage}{0.48\linewidth}
  \begin{equation}
  \label{eq:kernel_weight}
  \footnotesize
  \boxed{
  \begin{aligned}
    \textcolor{weightblue}{\omega_j^K(x;h)}
    &=
    \underbrace{
    \frac{
      \textcolor{weightblue}{
      K\!\left(\rho(x,X_j)/h\right)
      }
    }{
      \sum_{q=1}^{n}
      K\!\left(\rho(x,X_q)/h\right)
    }}_{\text{local normalized similarity}}
  \end{aligned}
  }
  \end{equation}
\end{minipage}
\begin{minipage}{0.52\linewidth}
  \begin{equation}
  \label{eq:kernel_event_times}
  \small
  \boxed{
  \begin{aligned}
    \mathcal{Y}_K(x;h)
    &=
    \{Y_j:
    \underbrace{
    \textcolor{eventred}{\delta_j=1}
    }_{\textcolor{eventred}{\substack{\text{event}\\\text{observed}}}},
    \;
    \underbrace{
    \textcolor{weightblue}{\omega_j^K(x;h)>0}
    }_{\textcolor{weightblue}{\substack{\text{locally}\\\text{weighted}}}}
    \}
  \end{aligned}
  }
  \end{equation}
\end{minipage}
\vspace{0.5em}

\noindent
\begin{minipage}{0.48\linewidth}
  \begin{equation}
  \label{eq:kernel_death_count}
  \footnotesize
  \boxed{
  \begin{aligned}
    \textcolor{eventred}{d_K(t\mid x;h)}
    &=
    \sum_{j=1}^{n}
    \underbrace{
    \textcolor{weightblue}{\omega_j^K(x;h)}
    }_{\textcolor{weightblue}{\text{local weight}}}
    \,
    \underbrace{
    \textcolor{eventred}{\delta_j\ind\{Y_j=t\}}
    }_{\textcolor{eventred}{\text{death at }t}}
  \end{aligned}
  }
  \end{equation}
\end{minipage}
\begin{minipage}{0.52\linewidth}
  \begin{equation}
  \label{eq:kernel_risk_count}
  \small
  \boxed{
  \begin{aligned}
    \textcolor{riskgreen}{r_K(t\mid x;h)}
    &=
    \sum_{j=1}^{n}
    \underbrace{
    \textcolor{weightblue}{\omega_j^K(x;h)}
    }_{\textcolor{weightblue}{\text{local weight}}}
    \,
    \underbrace{
    \textcolor{riskgreen}{\ind\{Y_j\geq t\}}
    }_{\textcolor{riskgreen}{\text{at risk at }t}}
  \end{aligned}
  }
  \end{equation}
\end{minipage}
\vspace{0.5em}
\begin{equation}
\label{eq:kernel_survival_estimator}
\small
\boxed{
\begin{aligned}
  \widehat{S}_K(t\mid x;h)
  &=
  \prod_{s\in\mathcal{Y}_K(x;h),\,0<s\leq t}
  \left(
  1-
  \underbrace{
  \textcolor{hazardpurple}{
  \frac{
    \textcolor{eventred}{d_K(s\mid x;h)}
  }{
    \textcolor{riskgreen}{r_K(s\mid x;h)}
  }}
  }_{\textcolor{hazardpurple}{\text{local hazard increment}}}
  \right)
\end{aligned}
}
\end{equation}
\end{subequations}
Equation~\ref{eq:kernel_weight} defines the normalized kernel weight assigned to observation $j$ when estimating the survival curve for query point $x$. These weights are nonnegative and sum to one when the denominator is nonzero. Equation~\ref{eq:kernel_event_times} defines the set of locally contributing observed event times. Equation~\ref{eq:kernel_death_count} computes the weighted number of deaths at time $t$, while Eq.~\ref{eq:kernel_risk_count} computes the weighted number of individuals still at risk just before time $t$. Their ratio gives a local estimate of the hazard increment. Finally, Eq.~\ref{eq:kernel_survival_estimator} combines these local hazard increments through a Kaplan--Meier product to produce the patient-specific survival estimate.
This estimator can be viewed as a localized Kaplan--Meier estimator. Instead of assigning equal importance to all training subjects, it emphasizes subjects that are close to the query point $x$ according to the distance function $\rho$. However, this dependence on a fixed distance metric can be problematic in high-dimensional clinical data, where raw-feature distances may not accurately represent similarity in survival behavior.
\subsection{Adaptive Nearest Neighbor through Tree Ensemble}
The kernel survival estimator in Eq.~\ref{eq:kernel_survival_estimator} defines locality using the fixed distance function $\rho$ and the bandwidth $h$. Although this is intuitive, it can be restrictive in high-dimensional biomedical settings. Two patients may be close in the original feature space but have different survival outcomes, while two patients that are far apart in raw covariates may share similar risk profiles. This motivates replacing raw-feature proximity with a learned notion of proximity that is more aligned with survival outcomes.
Adaptive nearest-neighbor methods address this limitation using a tree ensemble. Let
$\mathscr{T}=\{\mathcal{T}_1,\mathcal{T}_2,\ldots,\mathcal{T}_{|\mathscr{T}|}\}$ denote a survival-forest ensemble trained on bootstrap samples
$\mathcal{B}_1,\mathcal{B}_2,\ldots,\mathcal{B}_{|\mathscr{T}|}$ drawn from the observed training data $\mathcal{D}_n$. Each tree partitions the feature space into terminal leaf regions. If a query point $x$ and a training observation $X_j$ repeatedly fall into the same terminal leaves across the forest, then the forest considers them similar. The adaptive nearest-neighbor estimator uses this repeated co-membership to construct a data-adaptive kernel weight.
The forest-induced adaptive weight used in our method is motivated by the adaptive nearest-neighbor interpretation of random forests.
In particular, \citet{lin2006random} showed that random forests can be viewed as adaptive nearest-neighbor procedures, where observations repeatedly sharing terminal nodes with a query point receive larger influence.
A closely related weighting view appears in quantile regression forests, where \citet{meinshausen2006quantile} express forest predictions through weights induced by terminal-node co-membership.
For right-censored outcomes, Random Survival Forests extend this idea to survival data by growing survival trees on bootstrap samples and aggregating tree-level survival estimates \citep{ishwaran2008random}.
More generally, forest-induced adaptive weights have also been used as local weighting functions for nonparametric estimation \citep{athey2019generalized}.
For tree $\mathcal{T}_k$, let $\mathcal{R}_k(x)$ denote the terminal-node region containing the query point $x$.
Thus, training observation $X_j$ is considered close to $x$ in tree $\mathcal{T}_k$ if $X_j\in\mathcal{R}_k(x)$.
Let $\nu_{k,j}$ denote the number of times observation $j$ appears in the bootstrap sample $\mathcal{B}_k$ used to grow tree $\mathcal{T}_k$.
The adaptive weight construction is summarized in Eq.~\ref{eq:ann_weight_panel}.
The first quantity is the bootstrap-weighted size of the terminal node containing $x$; the second is the tree-level contribution of observation $j$; the third averages this contribution across trees; and the final quantity is the adaptive kernel itself.

\begin{subequations}
\label{eq:ann_weight_panel}
\noindent
\begin{minipage}{0.5\linewidth}
\begin{equation}
\label{eq:ann_leaf_size}
\footnotesize
\boxed{
\begin{aligned}
L_k(x)
&=
\sum_{q=1}^{n}
\underbrace{
\nu_{k,q}
}_{\substack{\text{bootstrap}\\\text{multiplicity}}}
\,
\underbrace{
\ind\{X_q\in\mathcal{R}_k(x)\}
}_{\substack{\text{same leaf}\\\text{as }x}}
\end{aligned}
}
\end{equation}
\end{minipage}
\hfill
\begin{minipage}{0.5\linewidth}
\begin{equation}
\label{eq:ann_tree_weight}
\footnotesize
\boxed{
\begin{aligned}
a_{k,j}(x)
&=
\underbrace{
\nu_{k,j}
\ind\{X_j\in\mathcal{R}_k(x)\}
}_{\substack{\text{in-bag count}\\\text{inside leaf}}}
\,/\,
\underbrace{
L_k(x)
}_{\substack{\text{leaf}\\\text{mass}}}
\end{aligned}
}
\end{equation}
\end{minipage}
\noindent
\begin{minipage}{0.48\linewidth}
\begin{equation}
\label{eq:ann_average_weight}
\small
\boxed{
\begin{aligned}
A_j(x)
&=
|\mathscr{T}|^{-1}
\sum_{k=1}^{|\mathscr{T}|}
a_{k,j}(x)
\end{aligned}
}
\end{equation}
\end{minipage}
\hfill
\begin{minipage}{0.48\linewidth}
\begin{equation}
\label{eq:ann_adaptive_kernel}
\small
\boxed{
\begin{aligned}
\textcolor{proxorange}{\widehat{K}_{\mathrm{ANN}}(x,X_j)}
&=
\textcolor{proxorange}{A_j(x)}
\end{aligned}
}
\end{equation}
\end{minipage}
\end{subequations}

In Eq.~\ref{eq:ann_leaf_size}, $L_k(x)$ denotes the bootstrap-weighted size of the terminal node containing $x$ in tree $\mathcal{T}_k$.
Equation~\ref{eq:ann_tree_weight} defines the tree-level adaptive weight assigned to training observation $X_j$.
This weight is positive only when $X_j$ belongs to the same terminal node as $x$, and it is normalized by the bootstrap-weighted mass of that node.
Equation~\ref{eq:ann_average_weight} averages these tree-level weights over the forest, and Eq.~\ref{eq:ann_adaptive_kernel} defines the resulting adaptive forest-induced kernel $\widehat{K}_{\mathrm{ANN}}(x,X_j)$.
When $L_k(x)>0$ for all trees, the weights satisfy $\widehat{K}_{\mathrm{ANN}}(x,X_j)\geq 0$ and $\sum_{j=1}^{n}\widehat{K}_{\mathrm{ANN}}(x,X_j)=1$.
Thus, the forest transforms the original feature space into an outcome-aware neighborhood structure.
Once the adaptive weights have been obtained, the survival curve is estimated using the same weighted Kaplan--Meier principle as in the kernel estimator. The fixed kernel weights are replaced by the adaptive forest-induced weights.
The adaptive estimator therefore preserves the structure of the kernel survival estimator while changing the notion of locality. The standard kernel estimator relies on geometric closeness in the original covariate space, whereas the adaptive estimator relies on forest-induced closeness. This learned proximity is based on repeated co-membership in terminal nodes and can capture nonlinear interactions between covariates. Consequently, the adaptive nearest-neighbor estimator is better suited to high-dimensional clinical settings where fixed distance metrics may be unstable or poorly aligned with survival outcomes.
\section{Proposed Method}
\label{proposed_methods}
\begin{figure}[t]
    \centering
    \resizebox{\textwidth}{!}{
\begin{tikzpicture}[
  font=\sffamily,
  >={Stealth[length=3.4mm,width=2.8mm,round]},
  flow/.style={arrowcol,line width=1.2pt,->,line cap=round},
  dflow/.style={querypurple,line width=1.2pt,dash pattern=on 3pt off 2.4pt,->,line cap=round},
  panel/.style={rounded corners=7pt,draw=panelstroke,line width=1pt,fill=panelfill},
  title/.style={font=\sffamily\bfseries\large},
  eqbox/.style={rounded corners=4pt,inner sep=6pt,align=center,line width=0.9pt},
]
\begin{scope}[shift={(0,0)}]
  \node[panel,minimum width=2.7cm,minimum height=6.9cm] (pop) at (0,0) {};
  \foreach \x/\y/\c in {
     -0.75/2.55/peoplegray, 0.0/2.65/peopleblue, 0.78/2.5/peoplegray,
     -0.55/1.9/peopleblue, 0.25/1.85/peopleblue, 0.9/1.7/peoplegray,
     -0.85/1.25/peopleblue, -0.05/1.15/peoplegray, 0.7/1.1/peopleblue,
     -0.6/0.5/peoplegray, 0.55/0.45/peopleblue,
     -0.85/-0.2/peopleblue, 0.85/-0.35/peoplegray,
     -0.3/-0.9/peoplegray, 0.6/-1.05/peopleblue,
     -0.8/-1.6/peopleblue, 0.1/-1.7/peoplegray, 0.85/-1.65/peopleblue,
     -0.5/-2.3/peoplegray, 0.45/-2.35/peoplegray, -0.95/-2.4/peopleblue,
     0.2/-2.85/peopleblue, -0.7/-2.9/peoplegray}
     {\pic[fill=\c,scale=1.05] at (\x,\y) {person};}
  \node[circle,draw=querypurple,line width=1.1pt,minimum size=0.72cm,fill=iconpurplebg] (qpop) at (-0.05,-4.55) {};
  \pic[fill=querypurple,scale=1.18] at (-0.05,-4.55) {person};
\end{scope}
\begin{scope}[shift={(5.4,3.6)}]
  \node[title] at (0,2.25) {(i) Survival profiles};
  \node[panel,minimum width=4.5cm,minimum height=4.0cm] (base) at (0,0) {};
  \foreach \y/\bg/\fg in {1.45/iconbluebg/learnblue, 0.5/icongreenbg/learngreen,
                          -0.45/iconorangebg/learnorange, -1.4/iconpurplebg/learnpurple}
     {\node[rounded corners=4pt,fill=\bg,minimum size=0.95cm] at (-1.45,\y) {};}
  \draw[learnblue,line width=1pt] (-1.77,1.22) -- (-1.40,1.5) -- (-1.13,1.66);
  \foreach \p in {(-1.77,1.22),(-1.40,1.5),(-1.13,1.66)} \fill[learnblue] \p circle (0.055);
  \draw[learngreen,line width=1pt]
    (-1.45,0.78)--(-1.73,0.4) (-1.45,0.78)--(-1.17,0.4)
    (-1.73,0.4)--(-1.85,0.18) (-1.73,0.4)--(-1.61,0.18)
    (-1.17,0.4)--(-1.05,0.18) (-1.17,0.4)--(-1.29,0.18);
  \foreach \p in {(-1.45,0.78),(-1.73,0.4),(-1.17,0.4),(-1.85,0.18),(-1.61,0.18),(-1.29,0.18),(-1.05,0.18)}
    \fill[learngreen] \p circle (0.052);
  \foreach \p in {(-1.77,-0.18),(-1.77,-0.7)} \fill[learnorange] \p circle (0.052);
  \foreach \p in {(-1.45,-0.05),(-1.45,-0.45),(-1.45,-0.85)} \fill[learnorange] \p circle (0.052);
  \fill[learnorange] (-1.13,-0.45) circle (0.052);
  \foreach \a in {(-1.77,-0.18),(-1.77,-0.7)}
    \foreach \b in {(-1.45,-0.05),(-1.45,-0.45),(-1.45,-0.85)}
      \draw[learnorange,line width=0.5pt] \a -- \b;
  \foreach \b in {(-1.45,-0.05),(-1.45,-0.45),(-1.45,-0.85)}
      \draw[learnorange,line width=0.5pt] \b -- (-1.13,-0.45);
  \draw[learnpurple,line width=1pt]
    (-1.85,-1.08) .. controls (-1.55,-1.12) and (-1.30,-1.5) .. (-1.07,-1.72);
  \foreach \p in {(-1.75,-1.32),(-1.55,-1.2),(-1.40,-1.5),(-1.21,-1.38),(-1.30,-1.66)}
    \fill[learnpurple] \p circle (0.045);
  \foreach \y/\fg/\sa/\sb/\sc in {
      1.45/learnblue/0.30/0.14/-0.02,
      0.5/learngreen/0.28/0.08/-0.10,
      -0.45/learnorange/0.32/0.10/-0.06,
      -1.4/learnpurple/0.26/0.06/-0.12}{
     \begin{scope}[shift={(1.00,\y)}]
        \draw[arrowcol,line width=0.6pt,->] (-0.55,-0.35) -- (0.85,-0.35);
        \draw[arrowcol,line width=0.6pt,->] (-0.55,-0.35) -- (-0.55,0.42);
        \draw[baseline,line width=0.5pt,dash pattern=on 1.6pt off 1.4pt] (-0.55,-0.28) -- (0.8,-0.28);
        \draw[\fg,line width=1.1pt]
           (-0.5,\sa)--(-0.28,\sa)--(-0.28,\sb)--(-0.04,\sb)--(-0.04,\sc)
           --(0.2,\sc)--(0.2,-0.18)--(0.46,-0.18)--(0.46,-0.24)--(0.78,-0.24);
     \end{scope}
  }
\end{scope}
\begin{scope}[shift={(11,3.55)}]
  \node[title] at (0,2.05) {(ii) $m$-out-of-$M$};
  \def\colsep{0.52}
  \foreach \row/\fg/\bg/\yy/\pat in {
      0/learnblue/iconbluebg/1.45/{1,1,0,1,0},
      1/learngreen/icongreenbg/0.72/{1,1,1,0,1},
      2/learnorange/iconorangebg/-0.01/{1,1,1,0,1},
      3/learnpurple/iconpurplebg/-0.74/{1,1,0,1,0}}{
    \node[rounded corners=8pt,draw=panelstroke,line width=0.9pt,fill=panelfill,
          minimum width=3.25cm,minimum height=0.62cm] at (0,\yy) {};
    \foreach \i [count=\j from 0] in \pat {
       \pgfmathsetmacro\xx{(\j-2)*\colsep}
       \ifnum\i=1
          \node[circle,fill=\fg,minimum size=0.40cm] at (\xx,\yy) {};
          \pic at (\xx,\yy) {checkmark};
       \else
          \node[circle,draw=panelstroke,line width=0.9pt,fill=white,minimum size=0.40cm] at (\xx,\yy) {};
       \fi
    }
  }
  \node[rounded corners=6pt,fill=iconorangebg!40,minimum width=3.25cm,minimum height=0.6cm] (mofm) at (0.3,-1.42) {};
  \foreach \j/\c in {0/learngreen,1/learngreen,2/peoplegray,3/peoplegray,4/peoplegray}{
     \pgfmathsetmacro\xx{(\j-2)*\colsep}
     \draw[panelstroke,line width=0.6pt,dash pattern=on 1.6pt off 1.4pt] (\xx,-0.95) -- (\xx,-1.2);
     \pic[fill=\c,scale=1.12] at (\xx,-1.42) {person};
  }
\node[inner sep=0pt,fit={(-1.8,1.8) (1.8,-1.75)}] (mofmpanel) {};
\end{scope}
\begin{scope}[shift={(7.6,-3.5)}]
  \node[title] at (0,2.2) {ANN weights};
  \node[panel,minimum width=9.0cm,minimum height=3.9cm] (ann) at (0,0) {};
  \begin{scope}
     \clip[rounded corners=7pt] (-4.5,-1.95) rectangle (4.5,1.95);
     \fill[iconpurplebg!55] (-4.5,-1.95) rectangle (4.5,-0.75);
  \end{scope}
  \foreach \cx in {-3.6,-2.4,-1.2,0,1.2,2.4,3.6}{
     \pic[scale=0.9] at (\cx,0.55) {anntree};
  }
  \draw[learnpurple,line width=0.6pt,dash pattern=on 1.8pt off 1.6pt]
    (-3.6,0.4) .. controls (-3.2,-0.15) and (-2.7,-0.65) .. (-2.45,-1.05);
  \draw[learnpurple,line width=0.6pt,dash pattern=on 1.8pt off 1.6pt]
    (-2.4,0.4) .. controls (-2.1,-0.15) and (-1.8,-0.65) .. (-1.75,-1.05);
  \draw[learnpurple,line width=0.6pt,dash pattern=on 1.8pt off 1.6pt]
    (-1.2,0.4) .. controls (-1.0,-0.20) and (-0.95,-0.70) .. (-0.95,-1.05);
  \draw[learnpurple,line width=0.6pt,dash pattern=on 1.8pt off 1.6pt]
    (0,0.4)    .. controls (0,-0.20) and (0,-0.72) .. (0,-1.05);
  \draw[learnpurple,line width=0.6pt,dash pattern=on 1.8pt off 1.6pt]
    (1.2,0.4)  .. controls (1.0,-0.20) and (0.95,-0.70) .. (0.95,-1.05);
  \draw[learnpurple,line width=0.6pt,dash pattern=on 1.8pt off 1.6pt]
    (2.4,0.4)  .. controls (2.1,-0.15) and (1.8,-0.65) .. (1.75,-1.05);
  \draw[learnpurple,line width=0.6pt,dash pattern=on 1.8pt off 1.6pt]
    (3.6,0.4)  .. controls (3.2,-0.15) and (2.7,-0.65) .. (2.45,-1.05);
  \foreach \x/\c in {
    -3.9/peoplegray,
    -3.25/learnpurple!20,
    -2.6/learnpurple!35,
    -1.95/learnpurple!50,
    -1.3/learnpurple!70,
    -0.65/querypurple,
     0.0/querypurple,
     0.65/querypurple,
     1.3/learnpurple!70,
     1.95/learnpurple!50,
     2.6/learnpurple!35,
     3.25/learnpurple!20,
     3.9/peoplegray}{
     \pic[fill=\c,scale=1.14] at (\x,-1.45) {person};
  }
\end{scope}
\begin{scope}[shift={(17.8,0.1)}]
  \node[title,querypurple] at (0,2.55) {(iii) KESurv weights};
  \node[circle,draw=panelstroke,line width=1pt,minimum size=4.0cm,fill=panelfill] (kes) at (0,0) {};
  \foreach \r/\op in {1.65/12,1.25/20,0.85/32,0.5/55}{
     \fill[querypurple,opacity=\op*0.01] (0,0) circle (\r);
  }
  \pic[fill=querypurple,scale=1.5] at (0,0) {person};
  \foreach \x/\y/\c in {
      -0.85/0.7/learngreen, 0.5/0.95/learngreen, 1.05/0.25/learngreen,
      0.75/-0.7/learngreen, -0.55/-0.85/learngreen, -1.1/-0.05/learngreen,
      0.0/1.35/peoplegray, -1.35/0.75/peoplegray, 1.4/-0.55/peoplegray,
      0.15/-1.35/peoplegray, -0.95/-1.1/peoplegray, 1.3/0.9/peoplegray}{
     \pic[fill=\c,scale=1.05] at (\x,\y) {person};
  }
\end{scope}
\begin{scope}[shift={(24,0.1)}]
  \node[title,align=center] at (0,2.4) {Patient-specific\\survival};
  \node[panel,minimum width=3.3cm,minimum height=3.4cm] (pssc) at (0,0) {};
  \begin{scope}[shift={(-1.05,-1.2)}]
     \draw[arrowcol,line width=0.8pt,->] (0,0) -- (2.25,0);
     \draw[arrowcol,line width=0.8pt,->] (0,0) -- (0,2.45);
     \node[font=\footnotesize] at (-0.28,2.25) {$1.0$};
     \node[font=\footnotesize] at (-0.2,0.05) {$0$};
     \node[font=\footnotesize\itshape] at (2.2,-0.25) {$t$};
     \draw[baseline,line width=0.7pt,dash pattern=on 2pt off 1.8pt] (0,0.45) -- (2.15,0.45);

     \draw[querypurple,line width=1.3pt]
        (0.0,2.2)--(0.4,2.2)--(0.4,1.82)--(0.75,1.82)--(0.75,1.6)--(1.05,1.6)
        --(1.05,1.18)--(1.35,1.18)--(1.35,0.82)--(1.7,0.82)--(1.7,0.5)--(2.1,0.5)--(2.1,0.42);
  \end{scope}
\end{scope}
\node[eqbox,draw=learnblue,fill=iconbluebg!55,font=\scriptsize] (e1) at (5.4,0.4)
  {$\begin{gathered}
    D_{\ell}^{\mathrm{ABS}}(x,X_j)
    =\dfrac{1}{t_{\max}}\int_{0}^{t_{\max}}
    \\[1pt]
    \big|\widehat{S}_{\ell}(s\mid x)
      -\widehat{S}_{\ell}(s\mid X_j)\big|\,ds
  \end{gathered}$};

\node[eqbox,draw=learngreen,fill=icongreenbg!60,font=\scriptsize] (e2) at (11,0.4)
  {$\begin{gathered}
    I_{\epsilon,m}(x,X_j) =
    \\[1pt]
    \ind\!\left\{
      \textstyle\sum_{\ell=1}^{M}
      \ind\{D_{\ell}^{\mathrm{ABS}}(x,X_j)\le\epsilon\}
      \ge m\right\}
  \end{gathered}$};

\node[eqbox,draw=learnorange,fill=iconorangebg!55,font=\scriptsize] (e3) at (7.7,-6.55)
  {$\widehat{K}_{\mathrm{ANN}}(x,X_j)
    =\dfrac{1}{|\mathscr{T}|}
      \textstyle\sum_{k=1}^{|\mathscr{T}|}a_{k,j}(x)$};

\node[eqbox,draw=querypurple,line width=1.2pt,fill=iconpurplebg,font=\scriptsize] (e4) at (17.7,-3.3)
  {$\begin{gathered}
    \widetilde{K}^{\mathrm{KES}}_{\epsilon,m}(x,X_j)
    =\widehat{K}_{\mathrm{ANN}}(x,X_j)I_{\epsilon,m}(x,X_j)
    \\[2pt]
    Z_{\epsilon,m}(x)
    =\textstyle\sum_{q=1}^{n}
      \widetilde{K}^{\mathrm{KES}}_{\epsilon,m}(x,X_q)
    \\[2pt]
    \widehat{K}^{\mathrm{KES}}_{\epsilon,m}(x,X_j)
    =\begin{cases}
      \widetilde{K}^{\mathrm{KES}}_{\epsilon,m}(x,X_j)/Z_{\epsilon,m}(x),
      & Z_{\epsilon,m}(x)>0,\\
      \widehat{K}_{\mathrm{ANN}}(x,X_j),
      & Z_{\epsilon,m}(x)=0.
    \end{cases}
  \end{gathered}$};

\node[eqbox,draw=querypurple,fill=iconpurplebg!55,font=\scriptsize] (e5) at (24.0,-3.3)
  {$\begin{gathered}
    \widehat{S}_{\mathrm{KESurv}}(t\mid x;\epsilon,m) =
    \\[2pt]
    \displaystyle\prod_{\substack{
      s\in\mathcal{Y}_{\epsilon,m}^{\mathrm{KES}}(x)\\s\le t}}
    \left(1-\dfrac{d^{\mathrm{KES}}_{\epsilon,m}(s\mid x)}
      {r^{\mathrm{KES}}_{\epsilon,m}(s\mid x)}\right)
  \end{gathered}$};

\draw[flow]
  (pop.east)
  .. controls +(1.2,0.6) and +(-1.0,-0.2) ..
  (base.west);
\draw[dflow]
  (qpop.south)
  .. controls +(1.2,-1.0) and +(-1.5,0.2) ..
  (ann.west);
\draw[flow] (base.south) -- (e1.north);
\draw[flow] (e1.east) -- (e2.west);
\draw[flow] (e2.north) -- (mofmpanel.south);
\draw[flow]
  (pop.south east)
  .. controls +(0.9,-0.4) and +(-0.7,0.9) ..
  (ann.north west);
\draw[flow]
  (mofmpanel.east)
  .. controls +(1.4,0.0) and +(-1.2,1.0) ..
  (kes.north west);
\draw[flow]
  (ann.north east)
  .. controls +(1.2,0.8) and +(-1.2,-0.4) ..
  (kes.south west);
\draw[flow]
  (kes.east)
  --
  (pssc.west);
\end{tikzpicture}
}
    \caption{
    Overview of the proposed KESurv framework.
    Base survival learners compare subjects in survival-curve space, while the survival forest provides adaptive nearest-neighbor weights.
    }
    \label{fig:kesurv_overview}
\end{figure}
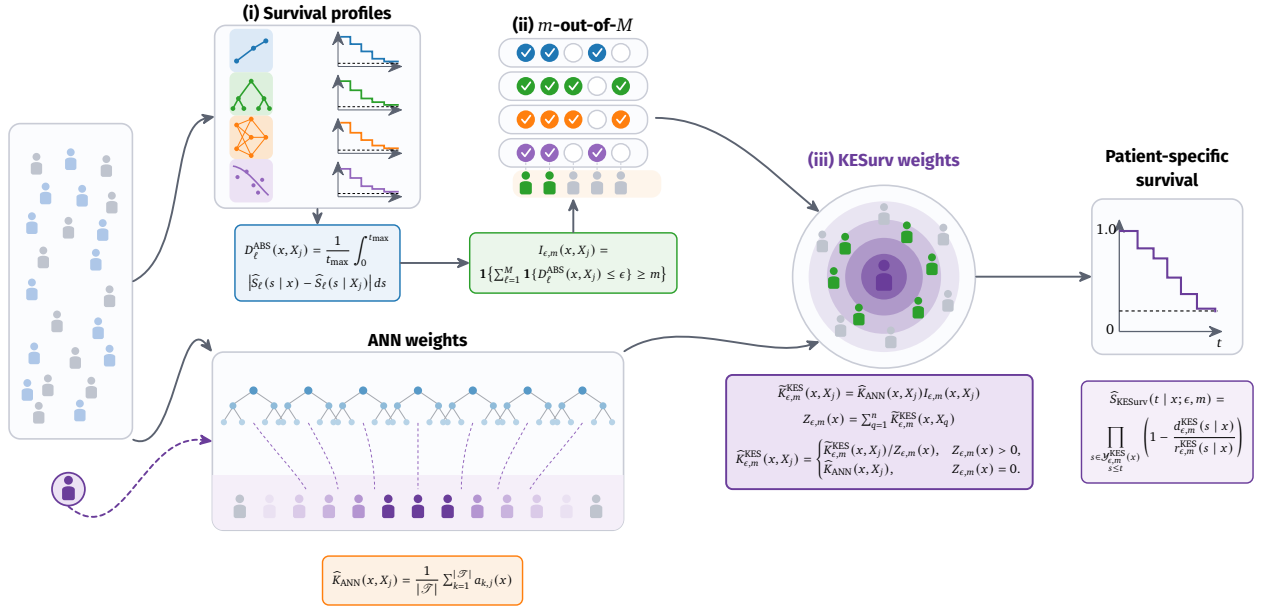
The objective of the proposed method, summarized in Figure~\ref{fig:kesurv_overview}, is to estimate a patient-specific survival curve for a query subject with covariates $x$.
The background section introduced two important ingredients: weighted Kaplan--Meier estimation and forest-induced adaptive nearest-neighbor weights.
The adaptive nearest-neighbor estimator uses the terminal-node structure of a survival forest to define a data-dependent notion of proximity.
However, a single adaptive kernel may still reflect only one representation of similarity.
In high-dimensional biomedical data, different survival learners may capture different aspects of patient risk, such as nonlinear feature interactions, proportional hazards effects, or local neighborhood structure.
This motivates a method that uses multiple survival learners to identify training observations whose predicted survival behavior is consistently similar to that of the query subject.
KESurv follows this idea.
Instead of averaging the survival predictions of several models, it uses the predictions of several base survival learners to define a more reliable neighborhood around the query point.
This is related in spirit to the Combined Regression Aggregation framework of \citet{biau2016cobra}, where base machines are used as a collective indicator of proximity rather than being combined through a standard linear ensemble.
In our survival setting, two subjects are considered close if their predicted survival curves are similar under a sufficient number of base survival learners.
The resulting consensus neighborhood is then combined with a survival-forest-induced adaptive kernel and used inside a weighted Kaplan--Meier estimator.
Let $\mathcal{G}=\{g_1,\ldots,g_M\} $ denote a collection of $M$ base survival learners. After training on the observed data $\mathcal{D}_n=\{(X_i,Y_i,\delta_i)\}_{i=1}^{n}$, the $\ell$-th learner $g_\ell$ produces an estimated survival curve
$
    \widehat{S}_{\ell}(t\mid x),
    \forall
    \ell=1,\ldots,M,
$
for a subject with covariates $x$.
The only requirement on a base learner is that it must be able to output survival probabilities over time.
Thus, the collection $\mathcal{G}$ may contain parametric, semiparametric, nonparametric, or machine-learning-based survival estimators.
The learners are trained independently, and their role in KESurv is to define survival-curve-based proximity rather than to directly form the final survival estimate.
\subsection{Survival-Curve Dissimilarity}
To use the base learners as proximity machines, we first need a way to compare the predicted survival curves of two subjects.
For two generic covariate vectors $u$ and $v$, the question is not whether they are close in the original feature space, but whether their predicted survival trajectories are close.
This distinction is important in high-dimensional clinical data, where raw covariate distance may be poorly aligned with survival behavior.
For the $\ell$-th learner, we define a survival-curve dissimilarity by measuring the normalized area between the two predicted survival curves.
This is motivated by ABS-type comparisons of survival functions, which quantify separation between survival curves through integrated absolute differences over time \citep{huang2022nonparametric}.
For a maximum follow-up time $t_{\max}$, the continuous form is
\begin{equation}
\label{eq:abs_distance_continuous}
D_{\ell}^{\mathrm{ABS}}(u,v)
=
t_{\max}^{-1}
\int_{0}^{t_{\max}}
\left|
\widehat{S}_{\ell}(s\mid u)
-
\widehat{S}_{\ell}(s\mid v)
\right|
\,ds
\end{equation}
Equation~\ref{eq:abs_distance_continuous} compares the entire survival trajectory predicted by learner $g_\ell$ for the two subjects $u$ and $v$.
The absolute value ensures that positive and negative deviations do not cancel out.
The factor $t_{\max}^{-1}$ normalizes the area by the length of the time interval, making the distance easier to compare across datasets with different follow-up ranges.
Since survival probabilities lie between 0 and 1, smaller values of $D_{\ell}^{\mathrm{ABS}}(u,v)$ indicate that the two subjects have similar predicted survival behavior under learner $g_\ell$.
In practice, survival curves are evaluated on a finite grid of time points.
Let $ 0=\tau_0<\tau_1<\cdots<\tau_P=t_{\max}$
denote the time grid used to evaluate the predicted survival curves.
Using a piecewise approximation, we compute the empirical survival-curve dissimilarity as
\begin{equation}
\label{eq:abs_distance_discrete}
D_{\ell}^{\mathrm{ABS}}(u,v)
=
\tau_P^{-1}
\sum_{b=1}^{P}
\left|
\widehat{S}_{\ell}(\tau_{b-1}\mid u)
-
\widehat{S}_{\ell}(\tau_{b-1}\mid v)
\right|
(\tau_b-\tau_{b-1})
\end{equation}
The discrete form in Eq.~\ref{eq:abs_distance_discrete} is the version used in the implementation.
The left-endpoint approximation is natural for Kaplan--Meier-type survival curves because such estimates are step functions.
The grid can be chosen as the set of observed event times, a fixed set of evaluation times, or any time grid shared across all base learners.
For consistency, the same grid is used when comparing all subjects under a given learner.
\subsection{Agreement-Based Survival Neighborhood}
The distance $D_{\ell}^{\mathrm{ABS}}(x,X_j)$ measures closeness between the query point $x$ and training observation $X_j$ according to learner $g_\ell$.
However, relying on only one learner may make the selected neighborhood sensitive to model-specific errors.
To improve stability, KESurv uses agreement across multiple survival learners.
The key idea is that a training observation should be considered a reliable neighbor of $x$ if several learners predict a survival curve for $X_j$ that is close to the predicted survival curve for $x$.
Importantly, KESurv does not require all $M$ learners to agree.
Instead, we introduce an agreement parameter $m\in\{1,\ldots,M\},$
where $m$ denotes the minimum number of survival learners that must consider a training observation close to the query point.
This makes the method flexible.
When $m=1$, the selector is permissive because agreement from any one learner is sufficient.
When $m=M$, the selector is strict because all learners must agree.
Intermediate values of $m$ provide a balance between these two extremes.
For each learner $\ell$ and threshold $\epsilon>0$, define the learner-specific survival neighborhood $
\mathcal{N}_{\epsilon,\ell}(x)
=
\{j:
D_{\ell}^{\mathrm{ABS}}(x,X_j)\leq \epsilon\} $, it contains the indices of training observations whose predicted survival curves are within radius $\epsilon$ of the query subject under learner $g_\ell$.
The parameter $\epsilon$ controls the tolerance for survival-curve dissimilarity.
A small value of $\epsilon$ keeps only subjects with highly similar predicted survival trajectories, while a larger value allows more subjects to be considered neighbors. 
The following selector uses the COBRA-style consensus structure of retaining an observation when sufficiently many base learners predict similar outputs \citep{biau2016cobra}. A survival-curve version of this construction appears in Area-norm COBRA \citep{goswami2023area}. Here, we use the selector to filter the forest-induced adaptive neighborhood, retaining its adaptive weights for survival aggregation. Using these learner-specific neighborhoods, we define the $m$-out-of-$M$ consensus selector
\begin{equation}
\label{eq:consensus_selector}
I_{\epsilon,m}(x,X_j)
=
\ind\left\{
\sum_{\ell=1}^{M}
\ind\{j\in \mathcal{N}_{\epsilon,\ell}(x)\}
\geq m
\right\},
\end{equation}
\noindent Equivalently, Eq.~\ref{eq:consensus_selector} can be written directly in terms of survival-curve distances as
$
I_{\epsilon,m}(x,X_j)
=
\ind\left\{
\sum_{\ell=1}^{M}
\ind\{D_{\ell}^{\mathrm{ABS}}(x,X_j)\leq \epsilon\}
\geq m
\right\}.
$
The selector equals one if at least $m$ of the $M$ survival learners regard $X_j$ as close to $x$.
Otherwise, it equals zero.
Thus, $I_{\epsilon,m}(x,X_j)$ acts as a survival-curve-based filter that keeps only those training observations supported by a sufficient number of base learners.
\subsection{KESurv Kernel}
The selector in Eq.~\ref{eq:consensus_selector} defines which training observations are retained, but it does not assign smooth weights among the selected observations.
To preserve the adaptive weighting structure introduced in the background section, we combine the selector with the forest-induced adaptive nearest-neighbor kernel.
Let $
    \widehat{K}_{\mathrm{ANN}}(x,X_j)$
denote the adaptive nearest-neighbor kernel obtained from the survival-forest ensemble, as defined in Eq.~\ref{eq:ann_adaptive_kernel}.
This kernel acts as a smoother because it gives graded weights according to forest-induced proximity.
The selector $I_{\epsilon,m}(x,X_j)$ acts as a gate because it removes observations that do not receive enough support from the base survival learners. We first define the unnormalized KESurv kernel:
\begin{equation}
\label{eq:proposed_kernel_unnormalized}
\widetilde{K}_{\epsilon,m}^{\mathrm{KES}}(x,X_j)
=
\underbrace{
\widehat{K}_{\mathrm{ANN}}(x,X_j)
}_{\substack{\text{smoother:}\\\text{forest proximity}}}
\,
\underbrace{
I_{\epsilon,m}(x,X_j)
}_{\substack{\text{selector:}\\m\text{-out-of-}M\text{ agreement}}}
\end{equation}
Equation~\ref{eq:proposed_kernel_unnormalized} is the central construction of KESurv.
The first component, $\widehat{K}_{\mathrm{ANN}}(x,X_j)$, measures how close $X_j$ is to $x$ according to the adaptive forest neighborhood.
The second component, $I_{\epsilon,m}(x,X_j)$, checks whether $X_j$ is also close to $x$ in predicted survival-curve space for at least $m$ base learners.
Therefore, a training observation receives positive weight only if it passes the survival-curve agreement rule.
Among the retained observations, the magnitude of the weight is still determined by the adaptive nearest-neighbor kernel.
Because the selector may set some weights to zero, the unnormalized weights in Eq.~\ref{eq:proposed_kernel_unnormalized} may not sum to one.
Although the weighted Kaplan--Meier ratio is invariant to a common scaling of all weights, normalizing the kernel makes the interpretation cleaner and avoids ambiguity.
We therefore define the normalizing mass
\begin{equation}
\label{eq:proposed_kernel_mass}
 Z_{\epsilon,m}(x)
=
\sum_{q=1}^{n}
\widetilde{K}_{\epsilon,m}^{\mathrm{KES}}(x,X_q),
\end{equation}
The normalized KESurv kernel is then defined as
\begin{equation}
\label{eq:proposed_kernel}
\small
\boxed{
\widehat{K}_{\epsilon,m}^{\mathrm{KES}}(x,X_j)
=
\begin{cases}
\widetilde{K}_{\epsilon,m}^{\mathrm{KES}}(x,X_j)/Z_{\epsilon,m}(x),
& Z_{\epsilon,m}(x)>0,\\[1mm]
\widehat{K}_{\mathrm{ANN}}(x,X_j),
& Z_{\epsilon,m}(x)=0.
\end{cases}
}
\end{equation}
The first case in Eq.~\ref{eq:proposed_kernel} is used when the agreement-based selector retains at least one training observation.
In that case, the selected adaptive weights are renormalized and form a valid kernel.
The second case is a fallback for the situation in which the rule is too restrictive and no observation satisfies the $m$-out-of-$M$ agreement condition.
Rather than producing an undefined estimate, KESurv reverts to the original adaptive nearest-neighbor kernel.
This fallback makes the estimator stable for small values of $\epsilon$ or large values of $m$.
\subsection{KESurv Survival Estimator}
Once the KESurv kernel has been defined, the survival curve is estimated using the same weighted Kaplan--Meier principle used for the kernel and adaptive nearest-neighbor estimators.
The only change is the choice of weights.
Instead of using the raw kernel weights or the original adaptive nearest-neighbor weights, KESurv uses the normalized weights $\widehat{K}_{\epsilon,m}^{\mathrm{KES}}(x,X_j)$. Let us first define the set of observed event times that receive positive KESurv weight:
\begin{equation}
\label{eq:proposed_event_times}
\small
\mathcal{Y}_{\epsilon,m}^{\mathrm{KES}}(x)
=
\{Y_j:
\delta_j=1,\;
\widehat{K}_{\epsilon,m}^{\mathrm{KES}}(x,X_j)>0\},
\end{equation}
The set $\mathcal{Y}_{\epsilon,m}^{\mathrm{KES}}(x)$ contains the event times associated with observations receiving positive final KESurv weight. When $Z_{\epsilon,m}(x)>0$ these are the selected observations; when $Z_{\epsilon,m}(x)=0$ they are determined by the fallback ANN kernel.
These are the times at which the KESurv estimator updates the survival probability for the query subject. For each time $t$, define the KESurv weighted death count and risk count as follows:
\begin{equation}
\label{eq:proposed_death_count}
d_{\epsilon,m}^{\mathrm{KES}}(t\mid x)
=
\sum_{j=1}^{n}
\widehat{K}_{\epsilon,m}^{\mathrm{KES}}(x,X_j)
\delta_j\ind\{Y_j=t\},
\end{equation}
\begin{equation}
\label{eq:proposed_risk_count}
\small
r_{\epsilon,m}^{\mathrm{KES}}(t\mid x)
=
\sum_{j=1}^{n}
\widehat{K}_{\epsilon,m}^{\mathrm{KES}}(x,X_j)
\ind\{Y_j\geq t\}
\end{equation}
Equation~\ref{eq:proposed_death_count} gives the KESurv weighted number of observed events at time $t$.
Only observations with $\delta_j=1$ and $Y_j=t$ contribute to this quantity.
Equation~\ref{eq:proposed_risk_count} gives the KESurv weighted number of observations still at risk just before time $t$.
The ratio
$
    d_{\epsilon,m}^{\mathrm{KES}}(t\mid x)
    /
    r_{\epsilon,m}^{\mathrm{KES}}(t\mid x)
$
is the local hazard increment induced by the proposed kernel. The final KESurv estimator is therefore
\begin{equation}
\label{eq:proposed_survival_estimator}
\widehat{S}_{\mathrm{KESurv}}(t\mid x;\epsilon,m)
=
\prod_{s\in\mathcal{Y}_{\epsilon,m}^{\mathrm{KES}}(x),\,0<s\leq t}
\left(
1-
\frac{d_{\epsilon,m}^{\mathrm{KES}}(s\mid x)
}{r_{\epsilon,m}^{\mathrm{KES}}(s\mid x)}
\right)
\end{equation}
Equation~\ref{eq:proposed_survival_estimator} shows that KESurv remains a weighted Kaplan--Meier estimator.
The novelty lies in how the weights are constructed.
The adaptive nearest-neighbor kernel contributes a smooth, forest-induced measure of proximity, while the $m$-out-of-$M$ selector contributes a survival-curve-based agreement condition.
Thus, the proposed method does not directly average the base survival predictions.
Instead, it uses them to identify a more reliable local neighborhood and then estimates the survival curve nonparametrically within that neighborhood.
Throughout, the superscript $\mathrm{KES}$ denotes the KESurv kernel weighting scheme, whereas $\mathrm{KESurv}$ denotes the resulting final survival estimator.
KESurv has two main hyperparameters: the distance threshold $\epsilon$ and the agreement level $m$. The threshold $\epsilon$ controls how close two predicted survival curves must be for a learner to consider them similar.
A smaller $\epsilon$ yields a more selective neighborhood, while a larger $\epsilon$ produces a broader neighborhood. The agreement parameter $m$ controls how many learners must support this similarity. A smaller value of $m$ is more permissive and may reduce variance by selecting more observations. A larger value of $m$ is more stringent and may reduce bias by retaining only observations with stronger cross-model agreement. These two parameters play complementary roles.
Thus, $\epsilon$ controls the radius of the survival-curve neighborhoods, while $m$ controls the level of consensus required across the ensemble.
The proposed estimator contains the adaptive nearest-neighbor survival estimator as a limiting special case. To see this, consider the limit $\epsilon\to\infty$. For any learner $\ell$ and any training observation $X_j$, the condition $D_{\ell}^{\mathrm{ABS}}(x,X_j)\leq \epsilon$
eventually holds. Therefore, for any fixed agreement level $m\in\{1,\ldots,M\}$,
$
    I_{\epsilon,m}(x,X_j)\to 1
    \qquad \forall j.
$
Consequently,
$
    \widehat{K}_{\epsilon,m}^{\mathrm{KES}}(x,X_j)
    \to
    \widehat{K}_{\mathrm{ANN}}(x,X_j),
$
and the KESurv survival estimator reduces to the adaptive nearest-neighbor survival estimator. For finite $\epsilon$, the method instead filters the adaptive neighborhood using survival-curve agreement across the base learners; KESurv can therefore be viewed as a survival-curve-guided refinement of adaptive nearest-neighbor estimation that recovers the latter as a special case.
Theoretical properties of the KESurv weights and a uniform consistency result under standard local-survival assumptions are provided in \appref{app:theory}.
\section{Experiments} \label{experiments}
We evaluate KESurv on four benchmark medical survival datasets to assess its ability to estimate patient-specific survival curves under different clinical settings.
The experiments are designed to compare KESurv against representative semiparametric, nonparametric, tree-based, and ensemble survival models using both discrimination and calibration-oriented metrics.
Unless otherwise stated, model selection is performed on a validation set, and final results are reported on held-out test data.
\subsection{Datasets}
To evaluate the performance of the proposed model, we considered four benchmark medical survival datasets: Flchain, Metabric, Support, and RoGBSG.
The Flchain dataset is derived from a study examining the association between serum free light chains and mortality, and contains demographic and clinical variables relevant to hematological disorders.
The Metabric dataset contains molecular and clinical information for breast cancer patients, including genetic profiles and survival outcomes, and has been widely used for studying molecular subtypes and treatment response.
The RoGBSG dataset combines the Rotterdam tumor bank cohort and the German Breast Cancer Study Group cohort.
Because the two cohorts share several common clinical variables, this combination is well suited to external validation.
This combined setting has been used in several recent survival-analysis studies.
Finally, the Support dataset consists of seriously ill hospitalized adults and includes information on diagnoses, treatment preferences, interventions, and survival outcomes.
Together, these datasets provide diverse evaluation settings across hematological disorders, breast cancer, and critical-care prognosis.
The main characteristics of the datasets are summarized in Table~\ref{tab:dataset_description}.
\begin{table}[htbp]
\centering
\renewcommand{\arraystretch}{1.15}
\setlength{\tabcolsep}{4pt}
\begin{tabular}{@{}lrrrcccccc@{}}
\toprule
\textbf{Dataset}
& $\boldsymbol{n}$
& $\boldsymbol{p}$
& \textbf{Cens. (\%)}
& \multicolumn{5}{c}{\textbf{Five-number summary of } $\boldsymbol{Y}$}
& \textbf{Ref.} \\
\cmidrule(lr){5-9}
& & &
& \textbf{Min}
& $\boldsymbol{Q_1}$
& \textbf{Median}
& $\boldsymbol{Q_3}$
& \textbf{Max}
& \\
\midrule
Flchain
& 6524 & 8  & 69.93
& 0.00 & 2907.50 & 4303.00 & 4771.00 & 5166.00
& \textsuperscript{a} \\
Metabric
& 1904 & 9  & 42.07
& 0.00 & 60.82 & 114.90 & 184.47 & 355.20
& \textsuperscript{b} \\
Support
& 8873 & 14 & 31.97
& 3.00 & 26.00 & 231.00 & 763.00 & 2029.00
& \textsuperscript{c} \\
RoGBSG
& 2232 & 7  & 43.23
& 0.26 & 19.12 & 40.23 & 71.10 & 87.36
& \textsuperscript{d} \\
\bottomrule
\end{tabular}
\caption{Dataset characteristics. Cens. denotes the percentage of censored observations $(\delta_i=0)$.}
\label{tab:dataset_description}
\vspace{1mm}
\begin{minipage}{\linewidth}
\footnotesize
\textsuperscript{a} \cite{dispenzieri2012use}.
\textsuperscript{b} \cite{rueda2019dynamics}.
\textsuperscript{c} \cite{connors1995controlled}.
\textsuperscript{d} Rotterdam tumor bank cohort \cite{foekens2000urokinase} and German Breast Cancer Study Group cohort \cite{schumacher1994randomized}; this combined setting has also been used in recent survival-analysis studies \cite{katzman2018deepsurv,kvamme2019time,chen2020deep,zhong2021deep,chen2024survival}.
\end{minipage}
\end{table}
\subsection{Experimental Setup} \label{sec:experimental_setup}
For each model and dataset, we repeated the experiment over 10 independent iterations.
In each iteration, the dataset was divided into training, validation, and test sets using a 70--15--15 split.
The training set was used to fit the models, while the validation set was used to select the best hyperparameter configuration for each method according to the search space reported in~\appref{app:hyperparam}.
After hyperparameter selection, the final performance of each model was evaluated on the held-out test set using the metrics described in \appref{app:metric_subsection}.
Across the 10 iterations, we report the mean and standard deviation of each evaluation metric.
The Rotterdam--GBSG experiment follows a different evaluation protocol.
In this setting, the experiment was not repeated over random splits; instead, the Rotterdam cohort was used as the training set and the GBSG cohort was used as the test set.
All models were trained on the Rotterdam data and evaluated directly on the GBSG data, providing an external validation setting across two independent cohorts.
Additional analyses rank-based statistical tests on the composite validation score (\appref{app:hyperparam}) and an ablation over the survival-curve distance metric (\appref{app:ablation}) are reported in the appendices.
The proposed KESurv model requires a set of base survival models to construct the ensemble-based survival similarity used in the proposed kernel.
In this work, we used four representative models: Cox Lasso (Lasso), Cox Ridge (Ridge), nearest-neighbor survival (NN), and a single survival tree (Tree), each described among the comparison models below.
Together, these models provide diverse survival estimates that KESurv uses to identify samples with similar predicted survival profiles. Within the training cohort, the reported experiments use a no-split construction: each base learner is fitted once on the training cohort, which also supplies the reference observations for survival aggregation. Consequently, stored training profiles are in-sample, whereas held-out query profiles are out-of-sample. Validation outcomes are used only for hyperparameter selection, and test outcomes are not used in fitting, tuning, or neighborhood construction. Theorem~1 in Appendix~\ref{app:theory} concerns the honest/cross-fitted construction under Assumption~3 and does not establish consistency of the reported no-split implementation.
\subsection{Comparison Models}
We compared KESurv with representative survival models from semiparametric, nonparametric, tree-based, and ensemble-based families.
The semiparametric baselines include Cox regression (Cox) and three penalized Cox variants: Lasso, Ridge, and Cox Elastic Net (Net).
These models estimate relative risk under the proportional hazards assumption, with the penalized variants using regularization to improve stability and control model complexity.
As nonparametric and tree-based baselines, we considered NN and Tree.
NN estimates patient-specific survival using locally similar training samples and the Beran estimator, while Tree recursively partitions the covariate space to obtain interpretable risk groups.
We also included Survival Forest (Forest), a tree-ensemble method for right-censored outcomes, and XGBoost, a gradient-boosting method adapted to survival prediction.
Some comparison models produce risk scores rather than direct survival probability estimates.
For these models, we used the Breslow estimator~\citep{lin2007breslow} to obtain survival curves, allowing all methods to be evaluated using the same survival-function-based metrics.
We did not include separate fixed-kernel survival baselines because nearest-neighbor and kernel survival estimators have been extensively studied in prior work~\citep{chen2019nearest}, and our focus is on evaluating whether the proposed ensemble-based survival-profile kernel improves patient-specific survival estimation.
\subsection{Results}
\begin{table*}[htbp]
    \centering
    \resizebox{0.90\textwidth}{!}{
    \begin{tabular}{lccccccccc}
        \toprule
        Dataset
        & Lasso
        & Net
        & Cox
        & Ridge
        & NN
        & Tree
        & XGBoost
        & Forest
        & KESurv \\
        \midrule
        \multirow{2}{*}{Flchain}
        & 0.7838
        & \textbf{0.7840}
        & 0.7835
        & 0.7835
        & 0.7480
        & 0.7672
        & 0.7825
        & 0.7688
        & 0.7814 \\
        & {\tiny($\pm$0.0101)}
        & {\tiny($\pm$0.0101)}
        & {\tiny($\pm$0.0102)}
        & {\tiny($\pm$0.0102)}
        & {\tiny($\pm$0.0114)}
        & {\tiny($\pm$0.0151)}
        & {\tiny($\pm$0.0104)}
        & {\tiny($\pm$0.0128)}
        & {\tiny($\pm$0.0128)} \\
        \addlinespace[0.35em]
        \multirow{2}{*}{Metabric}
        & 0.6281
        & 0.6283
        & 0.6276
        & 0.6278
        & 0.6378
        & 0.6208
        & 0.6407
        & 0.6490
        & \textbf{0.6599} \\
        & {\tiny($\pm$0.0198)}
        & {\tiny($\pm$0.0190)}
        & {\tiny($\pm$0.0196)}
        & {\tiny($\pm$0.0196)}
        & {\tiny($\pm$0.0188)}
        & {\tiny($\pm$0.0238)}
        & {\tiny($\pm$0.0145)}
        & {\tiny($\pm$0.0226)}
        & {\tiny($\pm$0.0172)} \\
        \addlinespace[0.35em]
        \multirow{2}{*}{RoGBSG}
        & 0.6222
        & 0.6192
        & 0.6172
        & 0.6172
        & 0.6526
        & 0.6423
        & 0.6219
        & 0.6501
        & \textbf{0.6560} \\
        & {\tiny($\pm$0.0000)}
        & {\tiny($\pm$0.0000)}
        & {\tiny($\pm$0.0000)}
        & {\tiny($\pm$0.0000)}
        & {\tiny($\pm$0.0000)}
        & {\tiny($\pm$0.0000)}
        & {\tiny($\pm$0.0000)}
        & {\tiny($\pm$0.0000)}
        & {\tiny($\pm$0.0000)} \\
        \addlinespace[0.35em]
        \multirow{2}{*}{Support}
        & 0.5969
        & 0.5966
        & 0.5965
        & 0.5965
        & 0.5993
        & 0.5974
        & 0.5982
        & 0.6221
        & \textbf{0.6425} \\
        & {\tiny($\pm$0.0080)}
        & {\tiny($\pm$0.0082)}
        & {\tiny($\pm$0.0082)}
        & {\tiny($\pm$0.0082)}
        & {\tiny($\pm$0.0081)}
        & {\tiny($\pm$0.0102)}
        & {\tiny($\pm$0.0086)}
        & {\tiny($\pm$0.0093)}
        & {\tiny($\pm$0.0108)} \\
        \bottomrule
    \end{tabular}
    }
    \caption{
    Time-dependent concordance results across four benchmark survival datasets.
    The best-performing method for each dataset is highlighted in bold.
    }
    \label{tab:concordance_table}
\end{table*}
\begin{table*}[htbp]
    \centering
    \resizebox{0.90\textwidth}{!}{
    \begin{tabular}{lccccccccc}
        \toprule
        Dataset
        & Lasso
        & Net
        & Cox
        & Ridge
        & NN
        & Tree
        & XGBoost
        & Forest
        & KESurv \\
        \midrule
        \multirow{2}{*}{Flchain}
        & 0.1036
        & 0.1035
        & 0.1034
        & 0.1034
        & 0.1067
        & 0.1056
        & 0.1333
        & 0.1033
        & \textbf{0.1027} \\
        & {\tiny($\pm$0.0040)}
        & {\tiny($\pm$0.0036)}
        & {\tiny($\pm$0.0039)}
        & {\tiny($\pm$0.0039)}
        & {\tiny($\pm$0.0055)}
        & {\tiny($\pm$0.0046)}
        & {\tiny($\pm$0.0047)}
        & {\tiny($\pm$0.0047)}
        & {\tiny($\pm$0.0045)} \\
        \addlinespace[0.35em]
        \multirow{2}{*}{Metabric}
        & 0.1658
        & 0.1658
        & 0.1659
        & 0.1658
        & 0.1672
        & 0.1730
        & 0.1749
        & 0.1690
        & \textbf{0.1637} \\
        & {\tiny($\pm$0.0099)}
        & {\tiny($\pm$0.0099)}
        & {\tiny($\pm$0.0102)}
        & {\tiny($\pm$0.0102)}
        & {\tiny($\pm$0.0057)}
        & {\tiny($\pm$0.0091)}
        & {\tiny($\pm$0.0071)}
        & {\tiny($\pm$0.0054)}
        & {\tiny($\pm$0.0063)} \\
        \addlinespace[0.35em]
        \multirow{2}{*}{RoGBSG}
        & 0.1763
        & 0.1758
        & 0.1756
        & 0.1756
        & 0.1717
        & 0.1734
        & 0.1975
        & 0.1723
        & \textbf{0.1679} \\
        & {\tiny($\pm$0.0000)}
        & {\tiny($\pm$0.0000)}
        & {\tiny($\pm$0.0000)}
        & {\tiny($\pm$0.0000)}
        & {\tiny($\pm$0.0000)}
        & {\tiny($\pm$0.0000)}
        & {\tiny($\pm$0.0000)}
        & {\tiny($\pm$0.0000)}
        & {\tiny($\pm$0.0000)} \\
        \addlinespace[0.35em]
        \multirow{2}{*}{Support}
        & 0.1952
        & 0.1951
        & 0.1951
        & 0.1951
        & 0.2007
        & 0.2002
        & 0.2151
        & \textbf{0.1913}
        & 0.1924 \\
        & {\tiny($\pm$0.0036)}
        & {\tiny($\pm$0.0036)}
        & {\tiny($\pm$0.0036)}
        & {\tiny($\pm$0.0036)}
        & {\tiny($\pm$0.0036)}
        & {\tiny($\pm$0.0057)}
        & {\tiny($\pm$0.0058)}
        & {\tiny($\pm$0.0048)}
        & {\tiny($\pm$0.0051)} \\
        \bottomrule
    \end{tabular}
    }
    \caption{
    Integrated Brier Score (IBS) results across four benchmark survival datasets.
    The best-performing method for each dataset is highlighted in bold.
    }
    \label{tab:results_table}
\end{table*}
The analysis of survival models, as outlined in Tables~\ref{tab:concordance_table} and \ref{tab:results_table}, reveals distinct performance profiles across datasets and evaluation metrics.  Net and Cox offer competitive performance, particularly on the Flchain dataset, where Net attains the best concordance. However, the proposed KESurv model distinguishes itself across datasets: KESurv achieves the best concordance on Metabric, RoGBSG, and Support, and the best IBS on Flchain, Metabric, and RoGBSG. On Support, Forest attains the lowest IBS, while KESurv obtains the best concordance. These results underscore the predictive reliability and robustness of KESurv.
Notably, KESurv performs consistently across datasets that differ in size, censoring rate, and clinical domain, whereas the baselines are strongest only on particular settings.
Because the adaptive nearest-neighbor kernel in KESurv is induced by a Survival Forest, the comparison with Forest most directly isolates what the survival-curve agreement selector contributes. KESurv improves the time-dependent concordance of Forest on all four datasets---for example, from $0.6221$ to $0.6425$ on Support and from $0.6490$ to $0.6599$ on Metabric---and it lowers the IBS on three of the four (Flchain, Metabric, and RoGBSG), with Forest marginally lower only on Support ($0.1913$ versus $0.1924$). Figure~\ref{fig:support_kesurv_vs_sf} shows that this advantage is stable across forest configurations: on Support, KESurv attains higher concordance in all $20$ depth and tree-count settings and lower IBS in $14$ of $20$. These gains indicate that filtering the forest-induced neighborhood by cross-model survival-curve agreement produces better-ranked patient-specific survival estimates, with lower integrated Brier error on most datasets, than the Survival Forest alone.
\begin{figure}[h]
    \centering
    \includegraphics[width=1\linewidth]{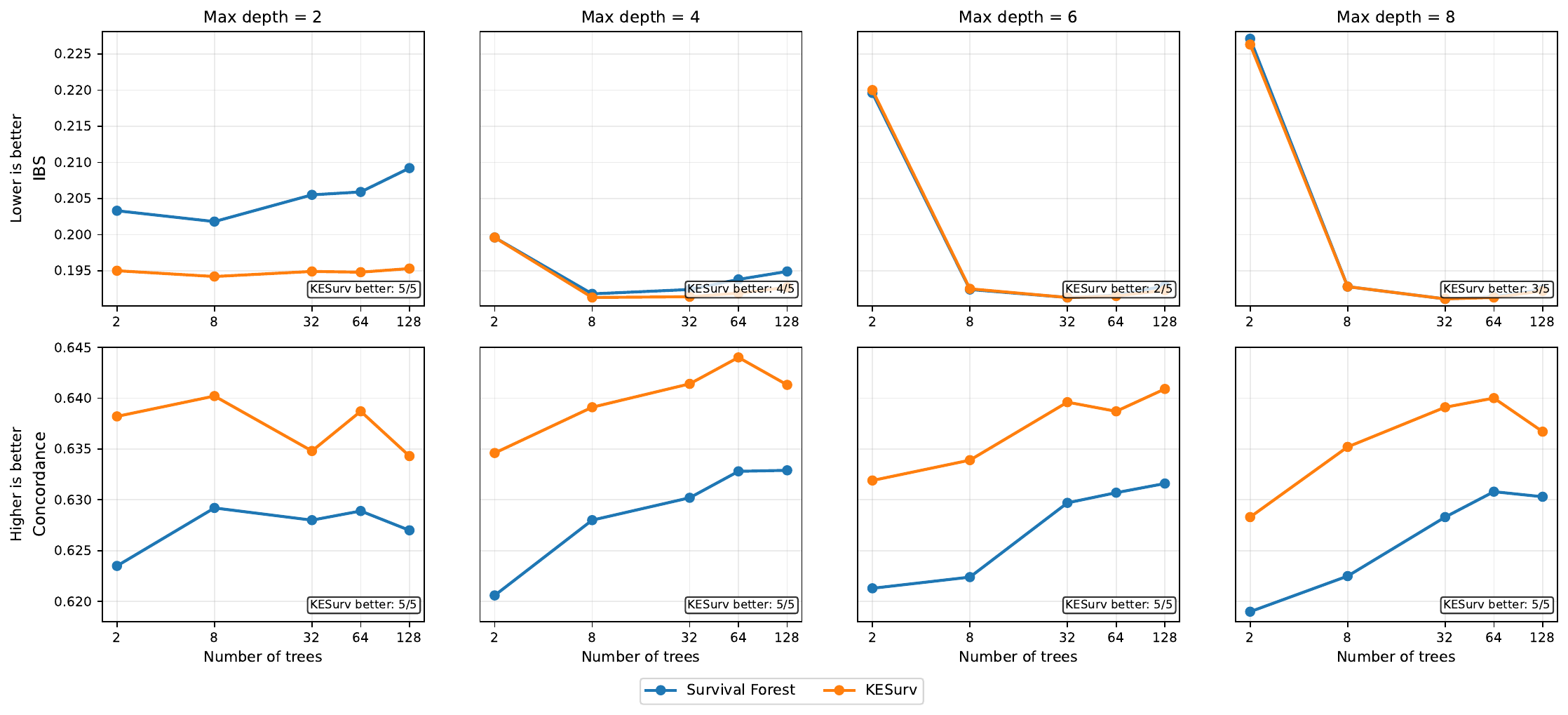}
\caption{
KESurv and Forest performance on the Support dataset across different maximum tree depths and numbers of trees.
The top row reports IBS, where lower is better, and the bottom row reports time-dependent concordance, where higher is better.
}
\label{fig:support_kesurv_vs_sf}
\vspace{-5.5mm}
\end{figure}
\section{Conclusion}
\label{conc}
In this work, we studied patient-specific survival prediction using machine learning methods for right-censored time-to-event data.
We proposed KESurv, an ensemble-based survival method that constructs a survival-function-aware neighborhood around each test sample.
Rather than relying only on covariate-space similarity, KESurv identifies neighboring patients whose predicted survival profiles are close to that of the query patient.
This combines the interpretability of neighborhood-based estimation with the flexibility of ensemble survival modeling.
Across the evaluated medical datasets, KESurv achieved competitive and often improved performance compared with standard survival models, indicating that survival-profile similarity can provide useful information for estimating patient-specific survival curves.
The method also offers a natural form of interpretability, since predictions can be examined through groups of patients with similar estimated survival behavior.
The main limitation of KESurv is computational cost.
Computing survival-function-based distances can be expensive when the number of patients, time points, or base models is large.
Although this cost can be partly reduced by precomputing and storing distance information, scalable storage and fast retrieval remain important practical challenges for large clinical datasets.
Future work should investigate efficient distance representations, approximate nearest-neighbor search, and subgroup discovery methods based on the proposed kernel.
These directions may improve both the scalability and interpretability of machine learning survival models in clinical and epidemiological applications.

\clearpage
\statmlappendix
\section{KESurv Algorithm and Complexity}
\label{sec:kesurv_algorithm} 
\paragraph{Algorithmic roadmap.}\enspace
KESurv proceeds in three stages:
\begin{enumerate}
    \item Obtain learner-specific survival profiles and compare
    the query with forest-supported training observations using
    $D_{\ell}^{\mathrm{ABS}}$.
    \item Convert these distances into learner-specific votes
    and apply the $m$-out-of-$M$ selector
    $I_{\epsilon,m}$.
    \item Multiply the selector by the forest-induced ANN
    weights, normalize using
    Eq.~\ref{eq:proposed_kernel}, and construct the resulting
    Beran (kernel-weighted Kaplan--Meier/product-limit)
    estimator.
\end{enumerate}
The base learners are fitted and the training survival profiles
are cached before processing held-out queries. Algorithm~\ref{alg:kesurv} describes prediction for a fitted model.
\begin{algorithm}[H]
\caption{Kernel Ensemble Survival (KESurv)}
\label{alg:kesurv}
\begin{algorithmic}[1]
\STATE \textbf{Input:} $\mathcal{D}_n=\{(X_i,Y_i,\delta_i)\}_{i=1}^{n}$, query point $x$, base learners $\mathcal{G}=\{g_1,\ldots,g_M\}$, ANN weights $\{\widehat{K}_{\mathrm{ANN}}(x,X_j)\}_{j=1}^{n}$, time grid $\{\tau_b\}_{b=0}^{P}$, threshold $\epsilon$, and agreement level $m$.
\STATE \textbf{Output:} $\widehat{S}_{\mathrm{KESurv}}(t\mid x;\epsilon,m)$.
\STATE Obtain the forest-supported set
$\mathcal{A}(x)=
\{j:\widehat{K}_{\mathrm{ANN}}(x,X_j)>0\}$.
\STATE Initialize
$\widetilde{K}_{\epsilon,m}^{\mathrm{KES}}(x,X_j)=0$
for all $j=1,\ldots,n$.
\FOR{$\ell=1,\ldots,M$}
    \STATE Use the fitted learner $g_\ell$ to obtain the query
    survival profile on the common time grid.
    \STATE Retrieve the cached training survival profiles and
    compute $D_{\ell}^{\mathrm{ABS}}(x,X_j)$ only for
    $j\in\mathcal{A}(x)$ using
    Eq.~\ref{eq:abs_distance_discrete}.
\ENDFOR
\FOR{$j\in\mathcal{A}(x)$}
    \STATE Compute the $m$-out-of-$M$ selector
    $I_{\epsilon,m}(x,X_j)$ using
    Eq.~\ref{eq:consensus_selector}.
    \STATE Compute the unnormalized KESurv weight
    $\widetilde{K}_{\epsilon,m}^{\mathrm{KES}}(x,X_j)$
    using Eq.~\ref{eq:proposed_kernel_unnormalized}.
\ENDFOR
\STATE Compute $Z_{\epsilon,m}(x)$ using Eq.~\ref{eq:proposed_kernel_mass}.
\STATE Normalize the weights to obtain $\widehat{K}_{\epsilon,m}^{\mathrm{KES}}(x,X_j)$ using Eq.~\ref{eq:proposed_kernel}.
\STATE Construct $\mathcal{Y}_{\epsilon,m}^{\mathrm{KES}}(x)$ using Eq.~\ref{eq:proposed_event_times}.
\STATE Compute $d_{\epsilon,m}^{\mathrm{KES}}(t\mid x)$ and $r_{\epsilon,m}^{\mathrm{KES}}(t\mid x)$ using Eqs.~\ref{eq:proposed_death_count}--\ref{eq:proposed_risk_count}.
\STATE Compute $\widehat{S}_{\mathrm{KESurv}}(t\mid x;\epsilon,m)$ using Eq.~\ref{eq:proposed_survival_estimator}.
\STATE \textbf{return} $\widehat{S}_{\mathrm{KESurv}}(t\mid x;\epsilon,m)$.
\end{algorithmic}
\end{algorithm}
\paragraph{Computational complexity}\enspace
Let $n$ denote the number of training samples, $M$ the number of base survival learners, and $P$ the number of time-grid points used to represent each survival curve.
For a single query point $x$, the main computational cost of KESurv comes from comparing the predicted survival curve of $x$ with the predicted survival curves of all training samples under each base learner.
If the survival predictions of the training samples have already been computed and stored, the survival-distance computation requires $O(MnP)$ time, since each of the $M$ learners compares $x$ with all $n$ training samples over $P$ time points.
The $m$-out-of-$M$ agreement selector and weight normalization require an additional $O(Mn)$ and $O(n)$ time, respectively.
Computing the weighted event and risk counts can be implemented naively in $O(nP)$ time, or more efficiently using sorted event times and cumulative sums.
Thus, ignoring the cost of fitting the base learners, the dominant per-query cost is typically $O(MnP)$, plus the cost of evaluating the ANN kernel $\widehat{K}_{\mathrm{ANN}}(x,X_j)$.
The latter depends on the forest implementation: traversing trees is usually efficient, but constructing the full set of nonzero forest-induced weights may still be expensive when many training samples share terminal nodes with the query point.
For $n_{\mathrm{test}}$ test samples, the direct prediction cost scales as $O(n_{\mathrm{test}}MnP)$, which can be substantial for large datasets, dense time grids, or many base learners.
The method also has a nontrivial memory cost if survival predictions or survival-distance matrices are precomputed; storing all base-learner survival curves for the training set requires $O(MnP)$ memory, while storing pairwise survival distances for many query--training pairs can become prohibitive.
Therefore, the current implementation is most practical for moderate-sized survival datasets, and scalable versions should consider sparse neighborhoods, approximate nearest-neighbor search, reduced time grids, low-dimensional survival-curve embeddings, or on-demand distance computation.
\paragraph{Forest-supported computation and storage}\enspace
Because
$\widetilde{K}_{\epsilon,m}^{\mathrm{KES}}(x,X_j)=0$
whenever $\widehat{K}_{\mathrm{ANN}}(x,X_j)=0$,
survival-curve distances need only be computed for
$j\in\mathcal{A}(x)$.
This reduces the distance-computation cost to
$O(M|\mathcal{A}(x)|P)$, with the benefit depending on
the sparsity of the forest neighborhood.
This bound excludes forest traversal, query-profile evaluation,
and construction of the final survival estimate.
Caching the training survival profiles requires $O(MnP)$
storage; storing a full query--training distance matrix is
optional and unnecessary.
Distances and agreement counts can instead be streamed,
using $O(n+P)$ additional working memory beyond the fitted
models and cached profiles.
The reported experiments do not include a controlled,
hardware-matched runtime or peak-memory benchmark.
Such profiling, together with approximate neighbor retrieval
and reduced-dimensional survival-profile representations,
remains future work.
\section{Bayesian Optimization and Hyperparameter Search Space} \label{app:hyperparam}
\subsection{Bayesian Optimization Strategy}
Hyperparameters were selected by minimizing a common validation objective across all models.
For a candidate hyperparameter configuration $\theta$, the objective function was defined as
\[
\mathcal{J}(\theta)
=
2\left(1-\mathrm{Concordance}(\theta)\right)
+
4\,\mathrm{IBS}(\theta),
\]
where $\mathrm{Concordance}(\theta)$ denotes the time-dependent concordance index $C_{td}$ and $\mathrm{IBS}(\theta)$ denotes the integrated Brier score.
Lower values of $\mathcal{J}(\theta)$ indicate better validation performance.
We used Bayesian optimization for hyperparameter tuning, as it provides a more efficient alternative to exhaustive search by using previous evaluations to guide the selection of future hyperparameter configurations \citep{snoek2012practical}.
In particular, we used the \texttt{Hyperopt} library with the Tree-structured Parzen Estimator (TPE) algorithm \citep{bergstra2011algorithms}.
TPE is a sequential model-based Bayesian optimization method that models promising and non-promising regions of the search space separately and proposes new configurations that are more likely to improve the objective function.
This is especially useful for our setting because evaluating each hyperparameter configuration requires training a survival model and computing validation performance using both time-dependent concordance and integrated Brier score.

\paragraph{Rotterdam--GBSG model selection.}\enspace
For the Rotterdam--GBSG experiment, Rotterdam was partitioned
into a training subset and a 15\% validation subset.
Each candidate configuration was fitted on the training subset
and evaluated on the validation subset using
$\mathcal{J}(\theta)$.
The selected fitted model was not refitted on the combined
Rotterdam cohort and was evaluated once on GBSG.
Neither GBSG covariates nor outcomes were used for fitting
or hyperparameter selection.
GBSG covariates were used only as query inputs at prediction
time, and its outcomes only for final evaluation.
\subsection{Hyperparameter Search Space}
\begin{table}[t]
\centering
\small
\renewcommand{\arraystretch}{1.15}
\begin{tabular}{p{3.4cm} p{10.2cm}}
\toprule
\textbf{Model} & \textbf{Search Space} \\
\midrule
Lasso
& $\alpha_{\min\_\mathrm{ratio}} \in \{0.01,0.02,0.05,0.07\}$ \\
Net
& $\mathrm{l1\_ratio}$ linearly spaced in $[0,1]$ with 20 values \\
Cox
& $\mathrm{ties} \in \{\mathrm{efron},\mathrm{breslow}\}$,
$n_{\mathrm{iter}} \in \{100,200,300,400\}$ \\
Ridge
& $\ell_2$ regularizer linearly spaced in $[0,1]$ with 20 values \\
NN
& $k_{\mathrm{NN}} \in \{10,15,20,25,\ldots,60\}$,
distance $\in \{\mathrm{Manhattan},\mathrm{Euclidean}\}$ \\
Tree
& $\mathrm{max\_features} \in \{\sqrt{p},\log_2(p),3,4,5,6,7,8\}$,
$\mathrm{min\_samples\_leaf} \in \{8,32,128\}$ \\
XGBoost
& $\lambda_{\mathrm{xgb}} \sim U(0,1)$, $\alpha_{\mathrm{xgb}} \sim U(0,1)$,
feature selector $\in \{\mathrm{cyclic},\mathrm{shuffle}\}$ \\
Forest
& $\mathrm{max\_features} \in \{3,4,5\}$,
$\mathrm{min\_samples\_leaf} \in \{2,4,8\}$ \\
KESurv (proposed)
& $m \in \{1,2,3,4\}$, $\epsilon \sim U(0,1)$ \\
\bottomrule
\end{tabular}
\caption{Hyperparameter search space optimized using Bayesian optimization with Hyperopt-TPE.}
\label{tab:search_space}
\end{table}
The tuned hyperparameters control regularization, model complexity, neighborhood structure, and ensemble-based survival similarity.
For penalized Cox models, the regularization parameters control sparsity and coefficient shrinkage.
For nearest-neighbor methods, $k_{\mathrm{NN}}$ and the distance metric determine the locality and similarity structure.
For tree-based models, $\mathrm{max\_features}$ controls split diversity, while $\mathrm{min\_samples\_leaf}$ controls terminal-node smoothness and prevents very small risk sets.
For XGBoost, the L1 and L2 penalties control sparsity and weight shrinkage, while the feature selector determines the coordinate update strategy.
For KESurv, $m$ is the minimum number of agreeing base
learners required to retain a training observation; all $M$
learners participate in the selector.
The parameter $\epsilon$ controls the survival-curve
dissimilarity threshold.
\begin{table*}[b]
    \centering
    \resizebox{\textwidth}{!}{
    \begin{tabular}{lccccccccc}
        \toprule
        Dataset & Lasso & Net & Cox & Ridge & NN & Tree & XGBoost & Forest & KESurv \\
        \midrule
        Flchain & 0.8468 & \textbf{0.8460} & 0.8466 & 0.8466 & 0.9308 & 0.8880 & 0.9682 & 0.8756 & 0.8480 \\
        Metabric & 1.4070 & 1.4066 & 1.4084 & 1.4076 & 1.3932 & 1.4504 & 1.4182 & 1.3780 & \textbf{1.3350} \\
        RoGBSG & 1.4608 & 1.4648 & 1.4680 & 1.4680 & 1.3816 & 1.4090 & 1.5462 & 1.3890 & \textbf{1.3596} \\
        Support & 1.5870 & 1.5872 & 1.5874 & 1.5874 & 1.6042 & 1.6060 & 1.6640 & 1.5210 & \textbf{1.4846} \\
        \bottomrule
    \end{tabular}
    }
    \caption{Composite criterion $\mathcal{J}(\theta)=2\,(1-\mathrm{Concordance}(\theta))+4\,\mathrm{IBS}(\theta)$ across four benchmark survival datasets. Lower is better; the best method per dataset is in bold.}
    \label{tab:J_table}
\end{table*}
To further compare the methods across datasets, we applied the Friedman test followed by the Nemenyi multiple-comparison procedure using the composite score $\mathcal{J}(\theta)$ reported in Table~\ref{tab:J_table}, where lower values indicate better performance.
The Friedman test rejected the null hypothesis of equal average ranks across methods ($p=0.024$), indicating statistically detectable differences among the evaluated models.
The corresponding Nemenyi multiple-comparison-with-the-best plot is shown in Figure~\ref{fig:nemenyi_mcb_j}.
KESurv achieved the lowest mean rank, followed by Forest and Net.
However, because the comparison is based on only four datasets, the critical distance is relatively large.
Therefore, the Nemenyi plot should be interpreted as supportive evidence that KESurv is among the strongest-performing methods, rather than as definitive evidence of pairwise superiority over every baseline.
\begin{figure}[htbp]
    \centering
    \includegraphics[width=0.6\textwidth]{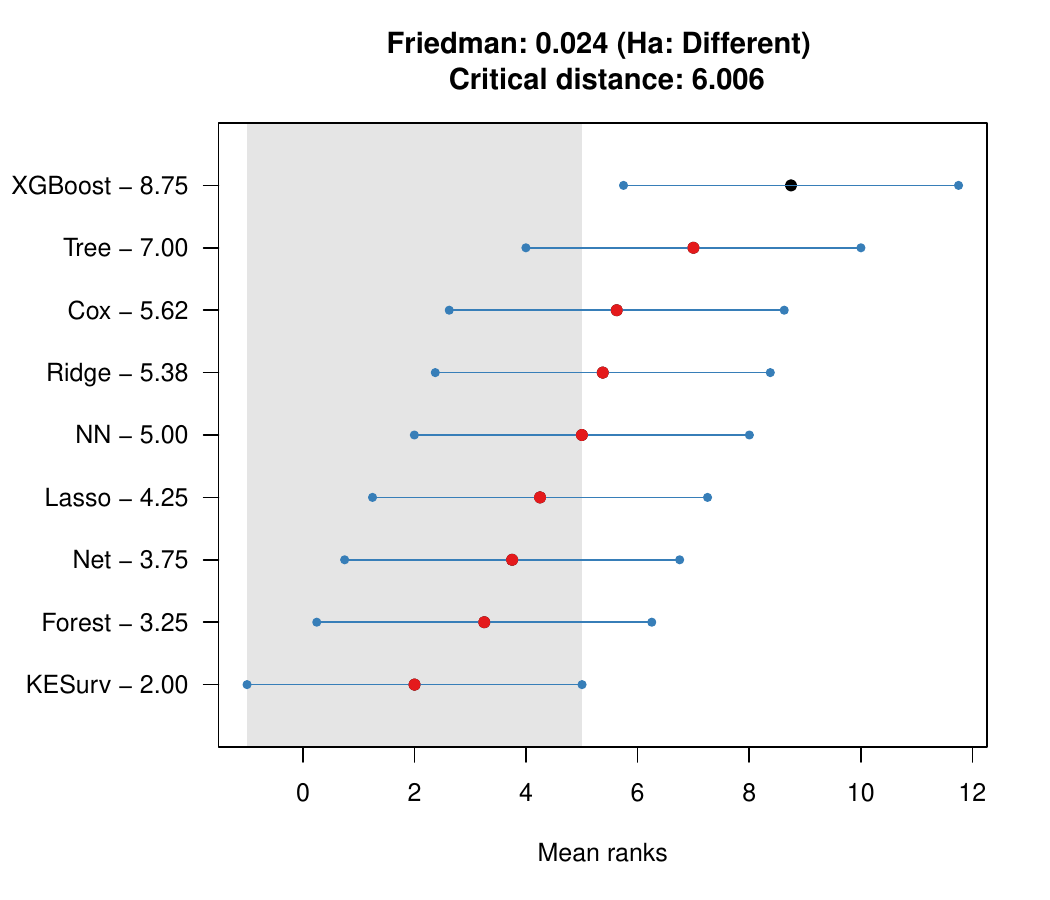}
    \caption{
    Nemenyi multiple-comparison-with-the-best plot based on the composite score
    $\mathcal{J}(\theta)=2(1-\mathrm{Concordance}(\theta))+4\,\mathrm{IBS}(\theta)$.
    Lower ranks indicate better overall performance.
    The Friedman test rejects equal average ranks across methods ($p=0.024$), and KESurv obtains the lowest mean rank.
    }
    \label{fig:nemenyi_mcb_j}
\end{figure}

\subsection{Base-Learner Eligibility and Diversity}

A base learner is eligible for KESurv if it outputs a valid
survival curve on the common time grid.
Eligibility alone does not establish usefulness.
In practice, we favor learners with useful prognostic signal
on validation data and nonredundant inductive biases.

The reported ensemble contains Cox Lasso, Cox Ridge,
nearest-neighbor survival, and a survival tree.
Lasso and Ridge share a proportional-hazards structure,
although their regularization differs; they should therefore
not be interpreted as independent model families.
Nearest-neighbor survival depends on the covariate metric
and local sample density, whereas a survival tree captures
nonlinearities and interactions through partition-based
neighborhoods.
Here, diversity refers to complementary modeling biases,
rather than statistical independence between predictions.
Because the learner predictions provide votes rather than
being averaged, including a Cox learner does not impose
proportional hazards on the final KESurv survival curve.

The $m$-out-of-$M$ rule retains a candidate supported by at
least $m$ learners despite disagreement from the remaining
$M-m$ learners.
Conversely, fewer than $m$ positive votes cannot admit a
candidate on their own.
This does not provide universal protection against a poor
learner: at $m=1$, one learner can admit a candidate, while
at $m=M$, one disagreeing learner can exclude it.
An uninformative or redundant collection can therefore harm
performance, and $m$ is selected using validation data.

The consistency conditions concern the resulting weighted
neighborhood: Assumptions 5--6 in Appendix~\ref{app:theory}
require local event-hazard homogeneity and a diverging
effective sample size.
These properties are not guaranteed merely by choosing
particular learner families.
\section{Performance Metrics and Calibration Results} \label{app:metric_subsection}
We employ two metrics, time-dependent concordance \cite{antolini2005time} and the Integrated Brier Score \cite{graf1999assessment}, and we additionally assess D-calibration \cite{haider2020effective}. For a given dataset $\{X_i,Y_i,\delta_i\}_{i=1}^n$, the time-dependent concordance index is defined as:
$$
C_{td} = \mathbb{P}(\widehat{S}(T_i \mid X_i) < \widehat{S}(T_i \mid X_j) \mid T_i < T_j, \delta_i = 1 ),
$$
that is, of two comparable subjects the one who experiences the event earlier is assigned the lower predicted survival at the shared comparison time $T_i$~\citep{antolini2005time}.
The Integrated Brier Score (IBS) quantifies the accuracy of predicted survival probabilities over a specified time interval, taking into account censoring. It is defined as:
\[
\mathrm{IBS} = \frac{1}{t_{\max}} \int_0^{t_{\max}} \mathrm{BS}(s) \, ds,
\]
 where the time-dependent Brier score $\mathrm{BS}(t)$ at time $t$ is defined as follows:
\begin{equation}
	\begin{split}
		\mathrm{BS}(t) = \frac{1}{n} \sum_{i=1}^{n} \bigg[ & \ind\{Y_i \leq t, \delta_i = 1\} \left(\frac{(0 - \widehat{S}(t \mid X_i))^2}{\widehat{G}(Y_i-)}\right)  + \ind\{Y_i > t\} \left(\frac{(1 - \widehat{S}(t \mid X_i))^2}{\widehat{G}(t)}\right) \bigg],
	\end{split}
\end{equation}
where $\widehat{S}(t \mid x)$ represents the predicted probability that an individual with feature vector $x$ remains event-free up to time $t$, and $\frac{1}{\widehat{G}(t)}$ is the inverse-probability-of-censoring weight, where $G(t)=\mathbb{P}(C\ge t)$ denotes the censoring survival function and $\widehat{G}(t)$ is its Kaplan--Meier estimate. In our experimentation, we have taken $t_{\max}$ as $\max_{i}Y_{i}$.
D-calibration is a critical method for evaluating whether patients can trust the survival curves predicted by a model. Given a dataset $\mathcal{D}_n$, a survival model function $\widehat{S}$, and an interval $[\gamma_0,\gamma_1] \subset [0,1]$, we define:
\[
\mathcal{I}_{\widehat{S}}([\gamma_0,\gamma_1]) = \{\, i : \delta_i = 1,\ \widehat{S}(Y_i \mid X_i) \in [\gamma_0,\gamma_1] \,\},
\]
The set
$\mathcal{I}_{\widehat{S}}([\gamma_0,\gamma_1])$
records the uncensored observations whose predicted survival
probability at the observed event time lies in the specified
interval.
For a correctly specified continuous survival distribution,
the probability transform $S(T\mid X)$ is uniformly
distributed on $[0,1]$.
Under censoring, an assessment based only on observed events
does not generally retain this uniformity; a
censoring-adjusted D-calibration procedure must also account
for censored observations, as described by
\citet{haider2020effective}.
Table~\ref{tab:results_dcalibration} reports the
D-calibration results.
\begin{table}[htbp]
    \centering
    \begin{tabular}{lrrrr}
        \toprule
        Models & Flchain & Metabric & RoGBSG & Support \\
        \midrule
        Lasso & 1.0000 & 0.7318 & 0.3881 & \textcolor{red}{0.0033} \\
        Net & 0.9999 & 0.7520 & 0.4058 & \textcolor{red}{0.0048} \\
        Cox & 1.0000 & 0.3968 & 0.3533 & \textcolor{red}{0.0034} \\
        Ridge & 1.0000 & 0.4267 & 0.3533 & \textcolor{red}{0.0032} \\
        NN & 0.4030 & 0.9658 & 0.4378 & \textcolor{red}{$<0.0001$} \\
        Tree & 0.9993 & 0.8271 & 0.4277 & \textcolor{red}{0.0109} \\
        XGBoost & \textcolor{red}{0.0001} & 0.5148 & \textcolor{red}{0.0306} & \textcolor{red}{$<0.0001$} \\
        Forest & 0.9993 & 1.0000 & 1.0000 & 0.8510 \\ \hline
        KESurv (proposed) & 0.9999 & 0.9996 & 0.5799 & 0.0807 \\
        \bottomrule
    \end{tabular}
    \caption{Aggregated D-calibration $p$-values across experiments with different random seeds, combined using Fisher's method \cite{fisher1925statistical}. Fisher's method computes the test statistic $\chi^2 = -2 \sum_{i=1}^{k} \ln(p_i)$, where $k$ represents the number of tests. The test did not reject D-calibration at the 5\% level. Red text identifies models that fail to meet this threshold ($p \le 0.05$).}
    \label{tab:results_dcalibration}
\end{table}
The D-calibration results further underscore the reliability of the KESurv model, with satisfactory calibration across all datasets, as evidenced by its ability to maintain p-values well above the 0.05 threshold, indicating reliable survival curve prediction. In contrast, models such as XGBoost and NN exhibit lower p-values in the Support dataset, which may suggest overfitting and less reliable predictions under certain conditions. IBS measures integrated survival-probability accuracy and
reflects both discrimination and calibration; it is not a
pure calibration metric.
We therefore distinguish time-dependent concordance for
discrimination, IBS for overall probabilistic accuracy,
and D-calibration for a dedicated calibration assessment.

\subsection{Scope of the Empirical Comparisons}

The evaluation examines whether survival-curve consensus
improves the matched Forest estimator and remains competitive
among the evaluated classical survival methods.
KESurv improves Forest concordance on all four datasets and
IBS on three of four; these comparisons do not establish
superiority over all survival-learning methods.

On FLCHAIN, Net achieves concordance 0.7840 compared with
0.7814 for KESurv, whereas KESurv achieves the lowest IBS,
0.1027.
The concordance difference of 0.0026 is small relative to
the reported across-split standard deviations, but this
does not establish statistical equivalence.
Model preference can therefore depend on whether
discrimination or survival-probability accuracy is the
primary objective.

Modern deep-survival models were not evaluated as comparators
or as KESurv base learners.
They could supply survival-profile votes when they provide
valid survival curves on the common time grid, but this
compatibility is not empirical evidence of improved
performance.
Evaluation against suitably tuned deep-survival baselines,
and their inclusion as base learners, remain future work.

\section{Ablation Study} \label{app:ablation}
In this section, we present an ablation study assessing the sensitivity of KESurv to the choice of survival-curve distance metric. The results are summarized in Table~\ref{tab:concordance_table_dist}.
\paragraph{Variations in Distance Metrics}\enspace
We initially employed the ABS distance metric, as defined in equation \ref{eq:abs_distance_continuous}, which yielded satisfactory results. However, to explore the impact of different distance metrics on model performance, we experimented with several alternatives.
\paragraph{Median Distance}\enspace
We defined the median survival time for a covariate vector $u$ under learner $g_\ell$, denoted $t_{\ell,0.5}(u)$, as the time at which the estimated survival function $\widehat{S}_{\ell}$ indicates a survival probability of 0.5. The distance between two covariate vectors $u$ and $v$ based on their median survival times is computed as follows:
\begin{equation}
	D_{\ell}^{\mathrm{Median}}(u, v) = \left| t_{\ell,0.5}(u) - t_{\ell,0.5}(v) \right|
\end{equation}
\paragraph{Maximum Distance}\enspace
To capture the maximum discrepancy between two survival curves, we defined the maximum distance as the greatest absolute difference across all time points, given by:
\begin{equation}
	D_{\ell}^{\mathrm{Max}}(u, v) = \max_{t} \left| \widehat{S}_{\ell}(t \mid u) - \widehat{S}_{\ell}(t \mid v) \right|
\end{equation}
\begin{table}[htbp]
	\centering
	\begin{tabular}{llcccc}
		\toprule
		& Distances & Flchain & Metabric & RoGBSG & Support \\
		\midrule
		\multirow{2}{*}{\rotatebox[origin=c]{90}{Con.}}
        & Median & 0.7786 \tiny{$\pm$0.0157} & 0.6544 \tiny{$\pm$0.0148} & 0.6509  & 0.6409 \tiny{$\pm$0.009} \\
        & Max & 0.7821 \tiny{$\pm$0.0131} & 0.6558 \tiny{$\pm$0.0198} & 0.6636  & 0.6415 \tiny{$\pm$0.0098} \\ \hline
        \multirow{2}{*}{\rotatebox[origin=c]{90}{IBS}}
        & Median & 0.1029  \tiny{$\pm$0.0048} & 0.1648  \tiny{$\pm$0.0056} &   0.1665   & 0.1919  \tiny{$\pm$0.0051} \\
        & Max &  0.1028  \tiny{$\pm$0.0046} &0.1658  \tiny{$\pm$0.009} &  0.1673  &  0.1926  \tiny{$\pm$0.0049} \\
		\bottomrule
	\end{tabular}
	\caption{The proposed model with performance for different distances}
	\label{tab:concordance_table_dist}
\end{table}
\section{Proofs of Theoretical Properties} \label{app:theory}
This appendix gives the theoretical details for the KESurv estimator. The arguments build on standard results from product-limit estimation, counting-process survival analysis, conditional Kaplan--Meier estimation, nearest-neighbor/kernel survival estimation, and forest-induced adaptive weighting. The weighted product-limit form follows the Kaplan--Meier estimator for right-censored data~\citep{kaplan1958nonparametric}. The martingale decomposition used in the consistency proof follows the counting-process formulation of censored survival data~\citep{andersen1982cox,fleming1991counting,andersen1993statistical}. The local consistency argument is related to Beran's conditional survival estimator and subsequent consistency results for kernel and nearest-neighbor conditional Kaplan--Meier estimators~\citep{beran1981nonparametric,dabrowska1987nonparametric,dabrowska1989uniform,akritas1994nearest,chen2019nearest}. The KESurv selector is motivated by the use of multiple base learners as a collective proximity rule, as in COBRA-type aggregation~\citep{biau2016cobra}. The adaptive smoothing component is motivated by the interpretation of random forests as adaptive nearest-neighbor methods~\citep{lin2006random}, forest-induced proximity weights~\citep{meinshausen2006quantile}, generalized random-forest weights~\citep{athey2019generalized}, and random survival forests for right-censored outcomes~\citep{ishwaran2008random,ishwaran2010consistency}. Our contribution is to combine these ingredients into a survival-curve-guided adaptive kernel, where a survival-forest weight is filtered by an $m$-out-of-$M$ agreement selector. For a fixed query point $x$, define \[ w_{n,j}(x) = \widehat{K}_{\epsilon,m}^{\mathrm{KES}}(x,X_j), \qquad j=1,\ldots,n . \] Let \[ N_j(t)=\ind\{Y_j\le t,\delta_j=1\}, \qquad R_j(t)=\ind\{Y_j\ge t\} \] denote the event-counting and at-risk processes for subject $j$. The corresponding weighted event and risk processes are \[ N_n^w(t\mid x) = \sum_{j=1}^{n}w_{n,j}(x)N_j(t), \qquad R_n^w(t\mid x) = \sum_{j=1}^{n}w_{n,j}(x)R_j(t). \] The KESurv estimator can then be written as the weighted product-limit estimator
\[ \widehat{S}_{\mathrm{KESurv}}(t\mid x;\epsilon,m) = \prod_{s\le t} \left\{ 1- dN_n^w(s\mid x)/R_n^w(s\mid x) \right\}, \] where terms with zero weighted risk are omitted. This is the usual Kaplan--Meier product-limit form with ordinary event and risk counts replaced by KESurv-weighted counts~\citep{kaplan1958nonparametric,beran1981nonparametric,chen2019nearest}.
\subsection{Assumptions for the Consistency Result}
The deterministic properties of the KESurv weights do not require additional probabilistic assumptions. For the consistency theorem, however, we analyze an honest or cross-fitted version of the estimator. This is standard when adaptive weights are learned from data, because it separates the construction of the local neighborhood from the event process used in the weighted Kaplan--Meier estimator. Similar honesty or sample-splitting ideas are commonly used in forest-based theory to control adaptivity~\citep{athey2019generalized}.
\begin{assumption}
The observations $(X_i,Y_i,\delta_i)$, $i=1,\ldots,n$, are independent and identically distributed, where
\[
    Y_i=\min(T_i,C_i),
    \qquad
    \delta_i=\ind\{T_i\le C_i\}.
\]
Here $T_i$ is the event time and $C_i$ is the censoring time.
\end{assumption}
\begin{assumption}
The event time and censoring time are conditionally independent given covariates:
\[
    T_i \perp C_i \mid X_i .
\]
This is the standard independent censoring condition used in conditional Kaplan--Meier and counting-process survival analysis~\citep{beran1981nonparametric,dabrowska1989uniform,andersen1993statistical}.
\end{assumption}

\paragraph{Interpretation of conditional independent censoring.}\enspace
The condition $T\perp C\mid X$ permits the censoring
distribution to depend on the observed covariates $X$;
it does not require marginal independence between $T$
and $C$.
It is also used in the theoretical development of
nearest-neighbor and kernel survival estimators
\citep{chen2019nearest,chen2024survival}.
This assumption cannot generally be verified from the
observed censored data alone.
Dependence between event and censoring times remaining
after conditioning on $X$, including dependence through
unmeasured factors, is a limitation of the analysis.
\begin{assumption}
For the purpose of the asymptotic argument, the KESurv weights $w_{n,j}(x)$ are constructed using covariates and an auxiliary training sample, or by a cross-fitting scheme, so that conditional on the covariates and the auxiliary sample, the weight assigned to subject $j$ does not depend on that subject's own event or censoring outcome. This allows the weighted counting-process martingale argument to be applied conditionally on the learned weights.
\end{assumption}
\begin{assumption} Fix a time horizon $\tau<t_{\max}$. For the query point $x$ and for every covariate value $u$ receiving positive KESurv weight, the conditional event-time distribution is absolutely continuous on $[0,\tau]$ with event hazard $\lambda(t\mid u)$. The event hazard is uniformly bounded: there exists $\lambda_{\max}<\infty$ such that
\[
    \sup_{0\le t\le \tau}
    \lambda(t\mid u)
    \le \lambda_{\max}
\]
for all such $u$ (including $u=x$). In addition, the conditional censoring survival function is positive on $[0,\tau]$: there exists $\eta>0$ such that
\[
    \mathbb{P}(C\ge \tau\mid X=u)\ge \eta
\]
for all covariate values receiving positive KESurv weight. The absolute-continuity condition implies that distinct subjects have no common event times with probability one, which is used in the martingale orthogonality and jump-size arguments.
\end{assumption}
\begin{assumption}
The selected and positively weighted observations are locally homogeneous with respect to the target event hazard at $x$:
\[
    \Delta_n(x)
    =
    \max_{j:w_{n,j}(x)>0}
    \sup_{0\le t\le \tau}
    \left|
    \lambda(t\mid X_j)-\lambda(t\mid x)
    \right|
    \xrightarrow{p}0 .
\]
This condition plays the same role as the shrinking-neighborhood or bandwidth condition in kernel and nearest-neighbor survival estimators~\citep{beran1981nonparametric,dabrowska1989uniform,akritas1994nearest,chen2019nearest}.
\end{assumption}
\begin{assumption}
The KESurv weights do not concentrate on finitely many observations:
\[
    \sum_{j=1}^{n}w_{n,j}^{2}(x)
    \xrightarrow{p}0 .
\]
Equivalently, the effective local sample size
\[
    n_{\mathrm{eff}}(x)
    =
    \left\{
    \sum_{j=1}^{n}w_{n,j}^{2}(x)
    \right\}^{-1}
\]
diverges in probability. This is analogous to standard local-smoothing requirements in nearest-neighbor and kernel survival estimation~\citep{akritas1994nearest,chen2019nearest}.
\end{assumption}
\begin{assumption}
There exists $c>0$ such that
\[
    \mathbb{P}\left(
    \inf_{0\le t\le \tau}
    R_n^w(t\mid x)
    \ge c
    \right)
    \longrightarrow 1 .
\]
This condition ensures that the weighted Kaplan--Meier estimator is evaluated on time intervals where the local weighted risk set is nondegenerate.
\end{assumption}
The assumptions above separate standard survival-estimation requirements (Assumptions 1, 2, 4, and 7) from KESurv-specific requirements (Assumptions 3, 5, and 6)~\citep{kaplan1958nonparametric,beran1981nonparametric,dabrowska1989uniform,athey2019generalized,fleming1991counting,andersen1993statistical,akritas1994nearest,chen2019nearest}. The consistency theorem should therefore be interpreted as follows: if the base survival learners and the ANN forest weights produce neighborhoods that become increasingly local in survival behavior while retaining a diverging effective sample size, then the KESurv weighted Kaplan--Meier estimator consistently estimates the target patient-specific survival curve. The theorem does not claim that KESurv uniformly dominates random survival forests or other survival estimators; rather, it establishes that the proposed weighting mechanism is compatible with standard nonparametric survival-estimation theory under appropriate locality, censoring, and effective-sample-size conditions.

\paragraph{Scope and bias--variance interpretation.}\enspace
Theorem~1 concerns the honest/cross-fitted construction
satisfying Assumption 3.
It does not establish consistency of the no-split
implementation used in the reported experiments. The proof separates cumulative-hazard estimation error
into local bias and martingale fluctuation.
Assumption 5 bounds the local bias by
$\tau\Delta_n(x)$, which vanishes as the positively weighted
neighborhood becomes homogeneous in event hazard.
Assumption 6 makes
$\sum_{j=1}^{n}w_{n,j}^{2}(x)$ vanish, while Assumption 7
keeps the weighted risk denominator nondegenerate.
Together, these conditions control the stochastic term.
They describe the required balance between locality and
effective sample size, rather than guaranteeing that every
choice of base learners or tuning parameters achieves it.
\subsection{Uniform Consistency}
Let
\[
    \Lambda(t\mid x)
    =
    \int_{0}^{t}
    \lambda(s\mid x)\,ds
\]
denote the conditional cumulative event hazard at the query point. Define the weighted Nelson--Aalen-type estimator
\[
    \widehat{\Lambda}_n^w(t\mid x)
    =
    \int_{0}^{t}
    dN_n^w(s\mid x)/R_n^w(s\mid x).
\]
This estimator is the weighted analogue of the Nelson--Aalen cumulative hazard estimator and is closely connected to the Kaplan--Meier product-limit estimator~\citep{fleming1991counting,andersen1993statistical}.
Throughout, the tuning parameters $(\epsilon,m)$ are treated as fixed or, more generally, as $n$-dependent choices $(\epsilon_n,m_n)$ for which the induced KESurv weights satisfy Assumptions 5--7.
\begin{theorem}[Uniform consistency of KESurv]
Under Assumptions 1--7,
\[
    \sup_{0\le t\le \tau}
    \left|
    \widehat{S}_{\mathrm{KESurv}}(t\mid x;\epsilon,m)
    -
    S(t\mid x)
    \right|
    \xrightarrow{p}
    0 .
\]
\end{theorem}
\begin{proof}
The proof follows the standard structure used for conditional Kaplan--Meier and kernel survival estimators: first establish uniform convergence of a local cumulative hazard estimator, then transfer this convergence through the product-limit map~\citep{beran1981nonparametric,dabrowska1989uniform,chen2019nearest}. Under conditional independent censoring, the counting-process representation gives
\[
    dN_j(t)
    =
    R_j(t)\lambda(t\mid X_j)\,dt
    +
    d\mathcal{M}_j(t),
\]
where $\mathcal{M}_j(t)$ is a martingale with respect to the observed filtration~\citep{andersen1982cox,fleming1991counting,andersen1993statistical}. By the honest or cross-fitted construction in Assumption 3, the weights may be treated as fixed predictable quantities when applying the martingale argument conditionally on the weight-generating information.
Let
\[
    \mathcal{E}_n
    =
    \left\{
    \inf_{0\le s\le \tau}
    R_n^w(s\mid x)
    \ge c
    \right\}.
\]
By Assumption 7, $\mathbb{P}(\mathcal{E}_n)\to 1$. On the event $\mathcal{E}_n$, the weighted risk denominator is uniformly positive, and substituting the counting-process decomposition into the weighted cumulative hazard estimator gives
\[
\begin{aligned}
    \widehat{\Lambda}_n^w(t\mid x)
    &=
    \int_{0}^{t}
    \frac{
    \sum_{j=1}^{n}
    w_{n,j}(x)dN_j(s)
    }{
    R_n^w(s\mid x)
    }                                                        \\
    &=
    \int_{0}^{t}
    \frac{
    \sum_{j=1}^{n}
    w_{n,j}(x)R_j(s)\lambda(s\mid X_j)
    }{
    R_n^w(s\mid x)
    }\,ds
    +
    \int_{0}^{t}
    \frac{
    \sum_{j=1}^{n}
    w_{n,j}(x)d\mathcal{M}_j(s)
    }{
    R_n^w(s\mid x)
    } .
\end{aligned}
\]
Subtracting the target cumulative hazard $\Lambda(t\mid x)$ yields
\[
    \widehat{\Lambda}_n^w(t\mid x)
    -
    \Lambda(t\mid x)
    =
    B_n(t)+\mathcal{M}_n^w(t),
\]
where
\[
    B_n(t)
    =
    \int_{0}^{t}
    \frac{
    \sum_{j=1}^{n}
    w_{n,j}(x)R_j(s)
    \{\lambda(s\mid X_j)-\lambda(s\mid x)\}
    }{
    R_n^w(s\mid x)
    }\,ds
\]
is the local bias term, and
\[
    \mathcal{M}_n^w(t)
    =
    \int_{0}^{t}
    \frac{
    \sum_{j=1}^{n}
    w_{n,j}(x)d\mathcal{M}_j(s)
    }{
    R_n^w(s\mid x)
    }
\]
is the martingale term. We first control the bias term. On $\mathcal{E}_n$ the denominator is positive. Because
\[
    R_n^w(s\mid x)
    =
    \sum_{j=1}^{n}
    w_{n,j}(x)R_j(s),
\]
we have, for all $t\le \tau$,
\[
\begin{aligned}
    |B_n(t)|
    &\le
    \int_{0}^{t}
    \frac{
    \sum_{j=1}^{n}
    w_{n,j}(x)R_j(s)
    \left|
    \lambda(s\mid X_j)-\lambda(s\mid x)
    \right|
    }{
    R_n^w(s\mid x)
    }\,ds                                                   \\
    &\le
    \int_{0}^{t}
    \Delta_n(x)\,ds                                         \\
    &\le
    \tau \Delta_n(x).
\end{aligned}
\]
By Assumption 5,
\[
    \Delta_n(x)\xrightarrow{p}0.
\]
Therefore,
\[
    \sup_{0\le t\le \tau}
    |B_n(t)|
    \xrightarrow{p}0 .
\]
We next control the martingale term. On $\mathcal{E}_n$ the predictable variation of $\mathcal{M}_n^w(t)$ is bounded by
\[
\begin{aligned}
    \left\langle \mathcal{M}_n^w\right\rangle(\tau)
    &\le
    c^{-2}
    \sum_{j=1}^{n}
    w_{n,j}^{2}(x)
    \int_{0}^{\tau}
    R_j(s)\lambda(s\mid X_j)\,ds .
\end{aligned}
\]
There are no cross-variation terms because, conditional on the weight-generating information, the subject-specific martingales are orthogonal, having no common event jumps under the absolute-continuity condition in Assumption 4.
By Assumption 4, the event hazard is uniformly bounded by $\lambda_{\max}$ on $[0,\tau]$. Thus
\[
    \int_{0}^{\tau}
    R_j(s)\lambda(s\mid X_j)\,ds
    \le
    \lambda_{\max}\tau .
\]
It follows that
\[
    \left\langle \mathcal{M}_n^w\right\rangle(\tau)
    \le
    c^{-2}\lambda_{\max}\tau
    \sum_{j=1}^{n}
    w_{n,j}^{2}(x).
\]
By Assumption 6,
\[
    \sum_{j=1}^{n}
    w_{n,j}^{2}(x)
    \xrightarrow{p}0,
\]
and hence
\[
    \left\langle \mathcal{M}_n^w\right\rangle(\tau)
    \xrightarrow{p}0 .
\]
Lenglart's domination inequality then implies
\[
    \sup_{0\le t\le \tau}
    |\mathcal{M}_n^w(t)|
    \xrightarrow{p}0
\]
conditionally on the weight-generating information, and therefore also unconditionally~\citep{lenglart1977relation,fleming1991counting,andersen1993statistical}. Combining the bias and martingale bounds gives
\[
    \sup_{0\le t\le \tau}
    \left|
    \widehat{\Lambda}_n^w(t\mid x)
    -
    \Lambda(t\mid x)
    \right|
    \xrightarrow{p}0 .
\]
It remains to transfer the result from cumulative hazards to survival functions. The KESurv estimator is the product-integral transform of $\widehat{\Lambda}_n^w$, while the target survival function is the product-integral transform of $\Lambda$~\citep{fleming1991counting,andersen1993statistical}. On $[0,\tau]$, $\Lambda(t\mid x)$ is finite because the event hazard is bounded. Moreover, on $\mathcal{E}_n$ the maximal jump of $\widehat{\Lambda}_n^w$ is bounded by
\[
    \sup_{s'\le \tau}
    \frac{
    dN_n^w(s'\mid x)
    }{
    R_n^w(s'\mid x)
    }
    \le
    c^{-1}
    \max_{1\le j\le n}
    w_{n,j}(x).
\]
The inequality uses the no-common-event-time condition in Assumption 4, which implies that at most one subject has an event at any event time with probability one. Since
\[
    \max_{1\le j\le n}
    w_{n,j}(x)
    \le
    \left\{
    \sum_{j=1}^{n}
    w_{n,j}^{2}(x)
    \right\}^{1/2}
    \xrightarrow{p}0,
\]
the jumps of the estimated cumulative hazard vanish. The continuity of the product-integral map under uniform convergence of cumulative hazards with vanishing jumps then yields
\[
    \sup_{0\le t\le \tau}
    \left|
    \widehat{S}_{\mathrm{KESurv}}(t\mid x;\epsilon,m)
    -
    S(t\mid x)
    \right|
    \xrightarrow{p}0 .
\]
This proves the theorem.
\end{proof}


\begin{thebibliography}{44}
\providecommand{\natexlab}[1]{#1}
\providecommand{\url}[1]{\texttt{#1}}
\expandafter\ifx\csname urlstyle\endcsname\relax
  \providecommand{\doi}[1]{doi: #1}\else
  \providecommand{\doi}{doi: \begingroup \urlstyle{rm}\Url}\fi

\bibitem[Aalen(1989)]{aalen1989linear}
Odd~O Aalen.
\newblock A linear regression model for the analysis of life times.
\newblock \emph{Statistics in medicine}, 8\penalty0 (8):\penalty0 907--925,
  1989.

\bibitem[Akritas(1994)]{akritas1994nearest}
Michael~G. Akritas.
\newblock Nearest neighbor estimation of a bivariate distribution under random
  censoring.
\newblock \emph{The Annals of Statistics}, 22\penalty0 (3):\penalty0
  1299--1327, 1994.

\bibitem[Andersen and Gill(1982)]{andersen1982cox}
Per~Kragh Andersen and Richard~D. Gill.
\newblock Cox's regression model for counting processes: A large sample study.
\newblock \emph{The Annals of Statistics}, 10\penalty0 (4):\penalty0
  1100--1120, 1982.

\bibitem[Andersen et~al.(1993)Andersen, Borgan, Gill, and
  Keiding]{andersen1993statistical}
Per~Kragh Andersen, {\O}rnulf Borgan, Richard~D. Gill, and Niels Keiding.
\newblock \emph{Statistical Models Based on Counting Processes}.
\newblock Springer, 1993.

\bibitem[Antolini et~al.(2005)Antolini, Boracchi, and
  Biganzoli]{antolini2005time}
Laura Antolini, Patrizia Boracchi, and Elia Biganzoli.
\newblock A time-dependent discrimination index for survival data.
\newblock \emph{Statistics in medicine}, 24\penalty0 (24):\penalty0 3927--3944,
  2005.

\bibitem[Athey et~al.(2019)Athey, Tibshirani, and Wager]{athey2019generalized}
Susan Athey, Julie Tibshirani, and Stefan Wager.
\newblock Generalized random forests.
\newblock \emph{The Annals of Statistics}, 47\penalty0 (2):\penalty0
  1148--1178, 2019.

\bibitem[Beran(1981)]{beran1981nonparametric}
Rudolf Beran.
\newblock Nonparametric regression with randomly censored survival data.
\newblock Technical report, University of California, Berkeley, 1981.

\bibitem[Bergstra et~al.(2011)Bergstra, Bardenet, Bengio, and
  K{\'e}gl]{bergstra2011algorithms}
James Bergstra, R{\'e}mi Bardenet, Yoshua Bengio, and Bal{\'a}zs K{\'e}gl.
\newblock Algorithms for hyper-parameter optimization.
\newblock In \emph{Advances in Neural Information Processing Systems},
  volume~24, 2011.

\bibitem[Biau et~al.(2016)Biau, Fischer, Guedj, and Malley]{biau2016cobra}
G{\'e}rard Biau, Aur{\'e}lie Fischer, Benjamin Guedj, and James~D Malley.
\newblock Cobra: A combined regression strategy.
\newblock \emph{Journal of Multivariate Analysis}, 146:\penalty0 18--28, 2016.

\bibitem[Chen(2019)]{chen2019nearest}
George Chen.
\newblock Nearest neighbor and kernel survival analysis: Nonasymptotic error
  bounds and strong consistency rates.
\newblock In \emph{International Conference on Machine Learning}, pages
  1001--1010. PMLR, 2019.

\bibitem[Chen(2020)]{chen2020deep}
George~H Chen.
\newblock Deep kernel survival analysis and subject-specific survival time
  prediction intervals.
\newblock In \emph{Machine Learning for Healthcare Conference}, pages 537--565.
  PMLR, 2020.

\bibitem[Chen(2024)]{chen2024survival}
George~H Chen.
\newblock Survival kernets: Scalable and interpretable deep kernel survival
  analysis with an accuracy guarantee.
\newblock \emph{Journal of Machine Learning Research}, 25\penalty0
  (40):\penalty0 1--78, 2024.

\bibitem[Chen and Guestrin(2016)]{chen2016xgboost}
Tianqi Chen and Carlos Guestrin.
\newblock Xgboost: A scalable tree boosting system.
\newblock In \emph{Proceedings of the 22nd acm sigkdd international conference
  on knowledge discovery and data mining}, pages 785--794, 2016.

\bibitem[Connors et~al.(1995)]{connors1995controlled}
Alfred~F Connors et~al.
\newblock A controlled trial to improve care for seriously ill hospitalized
  patients: The study to understand prognoses and preferences for outcomes and
  risks of treatments (support).
\newblock \emph{Jama}, 274\penalty0 (20):\penalty0 1591--1598, 1995.

\bibitem[Cox(1972)]{cox1972regression}
David~R Cox.
\newblock Regression models and life-tables.
\newblock \emph{Journal of the Royal Statistical Society: Series B
  (Methodological)}, 34\penalty0 (2):\penalty0 187--202, 1972.

\bibitem[Dabrowska(1987)]{dabrowska1987nonparametric}
Dorota~M. Dabrowska.
\newblock Non-parametric regression with censored survival time data.
\newblock \emph{Scandinavian Journal of Statistics}, 14\penalty0 (3):\penalty0
  181--197, 1987.

\bibitem[Dabrowska(1989)]{dabrowska1989uniform}
Dorota~M. Dabrowska.
\newblock Uniform consistency of the kernel conditional kaplan--meier estimate.
\newblock \emph{The Annals of Statistics}, 17\penalty0 (3):\penalty0
  1157--1167, 1989.

\bibitem[Dispenzieri et~al.(2012)]{dispenzieri2012use}
Angela Dispenzieri et~al.
\newblock Use of nonclonal serum immunoglobulin free light chains to predict
  overall survival in the general population.
\newblock \emph{Mayo Clinic Proceedings}, 87\penalty0 (6):\penalty0 517--523,
  2012.

\bibitem[Fisher(1925)]{fisher1925statistical}
Ronald~Aylmer Fisher.
\newblock \emph{Statistical Methods for Research Workers}.
\newblock Oliver and Boyd, Edinburgh, 1925.

\bibitem[Fleming and Harrington(1991)]{fleming1991counting}
Thomas~R. Fleming and David~P. Harrington.
\newblock \emph{Counting Processes and Survival Analysis}.
\newblock Wiley, 1991.

\bibitem[Foekens et~al.(2000)]{foekens2000urokinase}
John~A Foekens et~al.
\newblock The urokinase system of plasminogen activation and prognosis in 2780
  breast cancer patients.
\newblock \emph{Cancer research}, 60\penalty0 (3):\penalty0 636--643, 2000.

\bibitem[for the Molecular Classification~of
  Lung~Adenocarcinoma(2008)]{director2008gene}
Directors Challenge~Consortium for the Molecular Classification~of
  Lung~Adenocarcinoma.
\newblock Gene expression--based survival prediction in lung adenocarcinoma: a
  multi-site, blinded validation study.
\newblock \emph{Nature medicine}, 14\penalty0 (8):\penalty0 822--827, 2008.

\bibitem[Goswami and Dey(2023)]{goswami2023area}
Rahul Goswami and Arabin~Kr Dey.
\newblock Area-norm cobra on conditional survival prediction.
\newblock \emph{arXiv preprint arXiv:2309.00417}, 2023.

\bibitem[Graf et~al.(1999)Graf, Schmoor, Sauerbrei, and
  Schumacher]{graf1999assessment}
Erika Graf, Claudia Schmoor, Willi Sauerbrei, and Martin Schumacher.
\newblock Assessment and comparison of prognostic classification schemes for
  survival data.
\newblock \emph{Statistics in medicine}, 18\penalty0 (17-18):\penalty0
  2529--2545, 1999.

\bibitem[Haider et~al.(2020)Haider, Hoehn, Davis, and
  Greiner]{haider2020effective}
Humza Haider, Bret Hoehn, Sarah Davis, and Russell Greiner.
\newblock Effective ways to build and evaluate individual survival
  distributions.
\newblock \emph{Journal of Machine Learning Research}, 21\penalty0
  (85):\penalty0 1--63, 2020.

\bibitem[Harrell(2015)]{frank2015regression}
Frank~E Harrell.
\newblock Regression modeling strategies with applications to linear models,
  logistic and ordinal regression, and survival analysis, 2015.

\bibitem[Huang et~al.(2022)Huang, Lyu, Hou, and Chen]{huang2022nonparametric}
Xinghui Huang, Jingjing Lyu, Yawen Hou, and Zheng Chen.
\newblock A nonparametric statistical method for two crossing survival curves.
\newblock \emph{Communications in Statistics-Simulation and Computation},
  51\penalty0 (9):\penalty0 5041--5050, 2022.

\bibitem[Ishwaran and Kogalur(2010)]{ishwaran2010consistency}
Hemant Ishwaran and Udaya~B. Kogalur.
\newblock Consistency of random survival forests.
\newblock \emph{Statistics \& Probability Letters}, 80\penalty0
  (13--14):\penalty0 1056--1064, 2010.

\bibitem[Ishwaran et~al.(2008)Ishwaran, Kogalur, Blackstone, and
  Lauer]{ishwaran2008random}
Hemant Ishwaran, Udaya~B Kogalur, Eugene~H Blackstone, and Michael~S Lauer.
\newblock Random survival forests.
\newblock \emph{The Annals of Applied Statistics}, 2008.

\bibitem[Kalbfleisch and Prentice(2011)]{kalbfleisch2011statistical}
John~D Kalbfleisch and Ross~L Prentice.
\newblock \emph{The statistical analysis of failure time data}.
\newblock John Wiley \& Sons, 2011.

\bibitem[Kaplan and Meier(1958)]{kaplan1958nonparametric}
Edward~L. Kaplan and Paul Meier.
\newblock Nonparametric estimation from incomplete observations.
\newblock \emph{Journal of the American Statistical Association}, 53\penalty0
  (282):\penalty0 457--481, 1958.

\bibitem[Katzman et~al.(2018)]{katzman2018deepsurv}
Jared~L Katzman et~al.
\newblock Deepsurv: personalized treatment recommender system using a cox
  proportional hazards deep neural network.
\newblock \emph{BMC medical research methodology}, 18:\penalty0 1--12, 2018.

\bibitem[Kvamme et~al.(2019)Kvamme, Borgan, and Scheel]{kvamme2019time}
H{\aa}vard Kvamme, {\O}rnulf Borgan, and Ida Scheel.
\newblock Time-to-event prediction with neural networks and cox regression.
\newblock \emph{Journal of machine learning research}, 20\penalty0
  (129):\penalty0 1--30, 2019.

\bibitem[Lenglart(1977)]{lenglart1977relation}
E.~Lenglart.
\newblock Relation de domination entre deux processus.
\newblock \emph{Annales de l'I.H.P. Probabilit{\'e}s et statistiques},
  13\penalty0 (2):\penalty0 171--179, 1977.

\bibitem[Lin(2007)]{lin2007breslow}
D.~Y. Lin.
\newblock On the breslow estimator.
\newblock \emph{Lifetime data analysis}, 13:\penalty0 471--480, 2007.

\bibitem[Lin and Jeon(2006)]{lin2006random}
Yi~Lin and Yongho Jeon.
\newblock Random forests and adaptive nearest neighbors.
\newblock \emph{Journal of the American Statistical Association}, 101\penalty0
  (474):\penalty0 578--590, 2006.

\bibitem[Meinshausen(2006)]{meinshausen2006quantile}
Nicolai Meinshausen.
\newblock Quantile regression forests.
\newblock \emph{Journal of Machine Learning Research}, 7:\penalty0 983--999,
  2006.

\bibitem[Ranganath et~al.(2016)Ranganath, Perotte, Elhadad, and
  Blei]{ranganath2016deep}
Rajesh Ranganath, Adler Perotte, No{\'e}mie Elhadad, and David Blei.
\newblock Deep survival analysis.
\newblock In \emph{Machine Learning for Healthcare Conference}, pages 101--114.
  PMLR, 2016.

\bibitem[Rueda et~al.(2019)]{rueda2019dynamics}
Oscar~M Rueda et~al.
\newblock Dynamics of breast-cancer relapse reveal late-recurring er-positive
  genomic subgroups.
\newblock \emph{Nature}, 567\penalty0 (7748):\penalty0 399--404, 2019.

\bibitem[Sargent et~al.(2005)]{sargent2005disease}
Daniel~J Sargent et~al.
\newblock Disease-free survival versus overall survival as a primary end point
  for adjuvant colon cancer studies: individual patient data from 20,898
  patients on 18 randomized trials.
\newblock \emph{Journal of Clinical Oncology}, 23\penalty0 (34):\penalty0
  8664--8670, 2005.

\bibitem[Schumacher et~al.(1994)]{schumacher1994randomized}
M~Schumacher et~al.
\newblock Randomized 2 x 2 trial evaluating hormonal treatment and the duration
  of chemotherapy in node-positive breast cancer patients. german breast cancer
  study group.
\newblock \emph{Journal of Clinical Oncology}, 12\penalty0 (10):\penalty0
  2086--2093, 1994.

\bibitem[Snoek et~al.(2012)Snoek, Larochelle, and Adams]{snoek2012practical}
Jasper Snoek, Hugo Larochelle, and Ryan~P. Adams.
\newblock Practical bayesian optimization of machine learning algorithms.
\newblock In \emph{Advances in Neural Information Processing Systems},
  volume~25, 2012.

\bibitem[Van De~Vijver et~al.(2002)]{van2002gene}
Marc~J Van De~Vijver et~al.
\newblock A gene-expression signature as a predictor of survival in breast
  cancer.
\newblock \emph{New England Journal of Medicine}, 347\penalty0 (25):\penalty0
  1999--2009, 2002.

\bibitem[Zhong et~al.(2021)]{zhong2021deep}
Qixian Zhong et~al.
\newblock Deep extended hazard models for survival analysis.
\newblock \emph{Advances in Neural Information Processing Systems},
  34:\penalty0 15111--15124, 2021.

\end{thebibliography}
\end{document}